\pdfoutput=1
\documentclass{article}

\PassOptionsToPackage{numbers,compress}{natbib}
\usepackage[main, final]{neurips_2026}

\usepackage[utf8]{inputenc}
\usepackage[T1]{fontenc}
\usepackage{booktabs}
\usepackage{graphicx}
\usepackage{amsfonts}
\usepackage{amsmath}
\usepackage{enumitem}
\usepackage{amssymb}
\usepackage{amsthm}
\usepackage{bm}
\usepackage{nicefrac}
\usepackage{microtype}
\usepackage{xcolor}
\usepackage{float}
\usepackage{subcaption}
\usepackage{algorithm}
\usepackage{algpseudocode}
\usepackage{url}
\usepackage[colorlinks=true,citecolor=blue,linkcolor=blue,urlcolor=blue]{hyperref}
\newtheorem{theorem}{Theorem}[section]
\newtheorem{lemma}[theorem]{Lemma}

\DeclareMathOperator{\sem}{{\mathrm{sem}}}
\DeclareMathOperator{\tex}{{\mathrm{tex}}}

\title{Learning When to Denoise: Asynchronous Schedules for Faster
Multi-Representation Diffusion Training}
\author{%
  Bingshuo Qian \qquad Xiang Cheng \\
  Department of Electrical and Computer Engineering \\
  Duke University \\
  \texttt{\{bingshuo.qian, xiang.cheng\}@duke.edu}
}

\begin{document}

\maketitle

\begin{abstract}
Multi-representation diffusion models can improve visual synthesis by denoising
complementary views of an image, but their performance depends critically on
the asynchronous schedule that determines when each representation is denoised.
We propose to learn this schedule. Our method formulates asynchronous flow
matching over multiple representation spaces and uses a schedule-corrected
objective that keeps each representation's local noising-time weights fixed as
the schedule changes. We instantiate the schedule with a flexible parametric
class that is convex and monotone by construction, and learn it using a fast
joint probe with less than \(1\%\) additional training compute. On ImageNet
\(256\times256\), the learned schedule substantially improves both convergence
speed and final quality under a matched \(675\)M-parameter XL backbone. With
AutoGuidance, our \(200\)-epoch model reaches FID \(1.05\), matching the
\(800\)-epoch SFD-XL baseline with \(4\times\) less training. Training to
\(600\) epochs achieves \textbf{FID 1.02, below the 1.04 of the
1B-parameter SFD-XXL while using fewer parameters.}
In the unguided setting, our \(200\)-epoch model reaches FID \(2.37\), already
below the best \(800\)-epoch SFD-XL result (\(2.54\)) at \(4\times\) less
training, and improves to FID \(2.14\) at \(600\) epochs.
Code is available at \url{https://github.com/bsq532087/LWD}.
\end{abstract}

\begin{figure}[H]
  \centering
  \includegraphics[width=\linewidth]{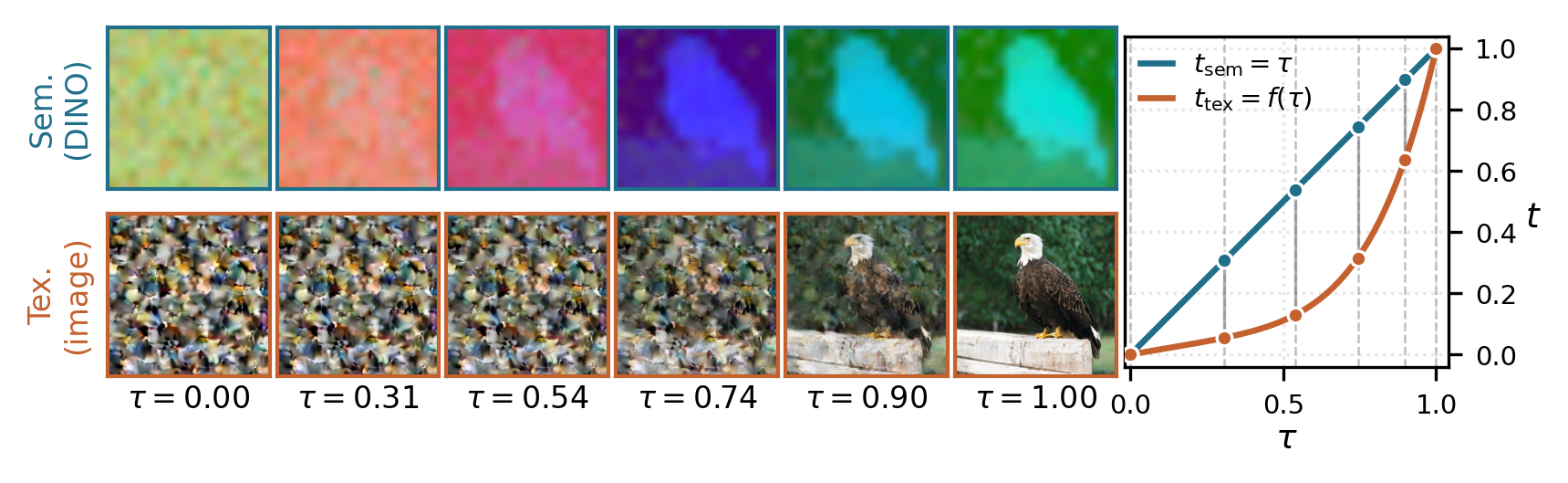}
  \caption{
    \textbf{Learning the relative denoising schedule between semantic and
    texture latents.}
    Along a shared global time \(\tau\), the semantic branch follows
    \(t_{\sem}=\tau\), while the texture branch follows a learned
    schedule \(t_{\tex}=f^\star(\tau)\le\tau\). Since larger \(t\)
    denotes a cleaner latent, semantic structure is formed earlier and guides
    subsequent texture refinement. The dashed lines mark the trajectory states
    visualized on the left.
  }
  \label{fig:intro-trajectory}
\end{figure}

\section{Introduction}
\label{sec:intro}
Diffusion and flow-based generative models have become a standard recipe for
high-fidelity visual synthesis. Much of their practical success comes from
choosing a representation in which denoising is both expressive and efficient:
latent diffusion reduces pixel-space cost by denoising compressed VAE latents
\citep{ldm}, while diffusion transformers and flow-matching variants scale this
recipe with large Transformer backbones \citep{dit,sit,flowmatching}.

A recent and complementary trend is to inject additional representational information into the generative process. For example, REPA aligns diffusion hidden states with clean features from pretrained visual encoders \citep{repa}, VA-VAE aligns the autoencoder latent space with vision-foundation-model features \citep{lightningdit}, and FAE compresses pretrained visual features into compact latents for generation \citep{fae}. SFD \citep{sfd} and Latent Forcing \citep{latentforcing} improved generation quality by directly denoising over multiple different representations.

The success of multi-representation denoising motivates a broader question: at
what rate should each representation be denoised? A synchronous schedule ties
all representations to the same noise level at each sampling step. An
asynchronous schedule instead allows different representations to occupy
different noise levels at the same step. This flexibility is useful because
different representations need not carry the same kind of information. A
higher-level semantic representation may encode global structure useful for
conditioning, while an image-decoding texture representation may be responsible
for recovering fine visual detail.

Choosing the asynchronous schedule is therefore a central design problem. A good
schedule must coordinate how information is revealed across representation
spaces, while balancing how easily the flow can be learned against how easily
the resulting ODE can be integrated. These requirements depend on the
representation pair, model architecture, training objective, and sampler. SFD
sets the schedule by hand: the semantic latent leads the texture latent by a
fixed offset chosen by grid search, which freezes the texture latent at the
start of sampling and the semantic latent at the end. Such a search does not
extend to more flexible schedule classes or more representations, where
selecting a good schedule becomes a difficult modeling problem rather than a
minor hyperparameter choice.

In this paper, we propose a framework for jointly learning the asynchronous schedule alongside the flow network. We first formulate a unified loss for both the flow network and the asynchronous schedule, and then propose an efficient two-stage algorithm that optimizes the schedule while learning to match the flow. Our key contributions are summarized below:

\paragraph{Contributions.}
\begin{enumerate}[leftmargin=0.4cm]
    \item \textbf{Asynchronous flow matching with learnable semantic--texture schedules.}
We formulate flow matching for composite latents with separate semantic and
texture representation groups, where each space follows its own local noising
time while a single global time indexes the sampling trajectory. This allows
general differentiable semantic--texture schedules. We theoretically characterize the ideal asynchronous flow as the conditional-mean vector field, which satisfies the continuity equation of the asynchronous interpolation, and equivalently as a
group-wise transformation of the score of that interpolation
(Theorem~\ref{thm:ideal-async-flow}). This formulation makes the schedule itself
an optimizable object, enabling flexible parameterizations beyond settings that can be selected by grid search.
\item \textbf{A schedule-corrected objective for flow and schedule learning.}
We define a component-wise flow objective that plays two roles. With the
schedule fixed, it learns the local asynchronous flow. With the denoiser
adapting, it provides the signal for learning the schedule. We show that its
population minimizer recovers the ideal local flows along the schedule
(Theorem~\ref{thm:async-fm-population-optimum}). For schedule learning, we
identify a confound: changing the schedule also changes the effective
distribution over local noising times at which the flow-matching losses are
evaluated. We prove a necessary and sufficient change-of-variables correction
that exactly counteracts this effect in the objective, keeping the marginal weighting over each
group's local noising time fixed as the schedule changes
(Lemma~\ref{lem:schedule-time-reweighting}). Finally, we add a kinetic-energy
regularizer that favors discretization-friendly trajectories, which completes
the schedule and denoiser objectives
(Eqs.~\eqref{eq:rho-objective} and~\eqref{eq:theta-objective}).

\item \textbf{Constrained schedule class and efficient joint optimization.}
We propose a parametric schedule class that enforces monotonicity, convexity,
and semantic-leading behavior by construction, while allowing efficient
evaluation of both the schedule and its derivatives. We also identify the need
for joint schedule--denoiser optimization: with a frozen denoiser, schedule
updates are dominated by model--schedule mismatch
(Section~\ref{sec:method:probe}). We therefore introduce an
efficient two-stage procedure (Algorithm \ref{alg:probe}) that jointly learns the schedule and the denoiser with less than \(1\%\) overhead in compute budget.

\item \textbf{Faster convergence and lower final FID.} On class-conditional
  ImageNet \(256\times256\), with SFD's architecture, latent representation,
  auxiliary losses, sampler, and weak-model architecture unchanged, the learned
  schedule reaches the final quality of SFD-XL with about \(4\times\) less
  training: unguided FID \(2.37\) at \(1\)M iterations versus \(2.54\) at
  \(4\)M, and \(1.05\) at \(200\) epochs versus \(1.06\) at \(800\) with
  AutoGuidance. It improves FID at every matched training budget, consistent
  with its smooth coupling being easier to learn than SFD's offset path with
  frozen intervals. At \(3\)M iterations (\(600\) epochs) it reaches FID
  \(2.14\) unguided and \(\mathbf{1.02}\) with AutoGuidance, the lowest among
  all \(675\)M-parameter models we compare against and below the
  \(1.0\)B-parameter SFD-XXL (\(1.04\)). The same schedule-learning procedure also improves over SFD when
  the SemVAE semantic latent is replaced by DINO-PCA or CLIP-PCA features,
  suggesting that the benefit comes from learning the semantic--texture schedule
  rather than from a representation-specific trick.
\end{enumerate}

\section{Related Work}
\label{sec:related}

\paragraph{Semantic-First Diffusion.}
The most directly comparable prior work is Semantic-First Diffusion (SFD)
\citep{sfd}, which also forms the basis of our experimental setup. SFD
encodes each image into two latent groups, a texture latent from an image
VAE and a semantic latent from a SemVAE compressing pretrained DINOv2
features, and trains a single diffusion transformer that denoises both
groups asynchronously, with the semantic group leading by a fixed offset. We
retain SFD's architecture, latent representation, weak-model architecture for
AutoGuidance, and sampler, but replace this offset with a learned
semantic-leading schedule.

\paragraph{Asynchronous and multi-representation denoising.}
Several other generative systems exploit the fact that different parts of a
sample need not be denoised at the same rate. Diffusion Forcing assigns
independent noise levels to different tokens, bridging next-token prediction
and full-sequence diffusion in sequential domains \citep{diffusionforcing}.
Latent Forcing reorders the diffusion trajectory across latents and pixels
with separately tuned schedules, letting latents form before high-frequency
pixel content \citep{latentforcing}. Like SFD, both fix the
cross-representation schedule by hand or by a small low-dimensional sweep.
Our work shares their goal of decoupled denoising
rates, but learns a flexible semantic-leading schedule from data while keeping
the representation setup and sampler fixed. Recent work also studies learned anisotropic schedules more broadly: \citet{liu2026anisotropic} optimize matrix-valued noise schedules that allocate noise across subspaces. Our focus is the representation-aware latent setting: we learn a semantic-leading schedule for explicit semantic and texture groups,
with a corrected objective that keeps local-time weighting invariant.

\paragraph{Representation-aware diffusion.}
A complementary line of work changes \emph{which} representations the
denoiser operates on. REPA aligns hidden states with pretrained visual
features \citep{repa}; REPA-E extends this to end-to-end joint training of
the VAE and diffusion model \citep{repae}; VA-VAE aligns the autoencoder
latent space with vision-foundation-model features \citep{lightningdit}; FAE
compresses pretrained visual features into compact generative latents
\citep{fae}; and further work incorporates semantic features through
representation entanglement, joint image--feature synthesis, or
representation autoencoding \citep{reg,redi,rae}. These methods improve \emph{what} the denoiser sees
under a single global denoising schedule. We are concerned instead with how
different representations should evolve over denoising time once a
multi-representation latent is given.

\section{Method}
\label{sec:method}

\subsection{Asynchronous Flow in Texture and Semantic Spaces}
\label{sec:method:prelim}

\paragraph{Semantic and texture spaces.}
We now consider a composite latent with two representation
groups. The \emph{texture} group is the image VAE latent: it is the latent that
is decoded directly to pixels at the end of sampling. The \emph{semantic} group
contains higher-level visual information, such as features derived from DINOv2.
In our main experiments, we follow Semantic-First Diffusion (SFD)~\citep{sfd}:
the texture group is the SFD image-VAE latent, and the semantic group is the
SemVAE latent trained to compress pretrained DINOv2 features~\citep{dinov2}.
We also evaluate alternative semantic representations, including DINO-PCA and
CLIP-PCA~\citep{clip}, in Section~\ref{sec:exp:repr-robustness}. We refer to
SFD~\citep{sfd} for additional details on the representation construction. Let
\begin{equation}
  x_1
  =
  \big[
    x_1^{\tex},
    x_1^{\sem}
  \big]\in \mathbb{R}^{d_{\tex}+d_{\sem}},
  \qquad
  x_0
  =
  \big[
    x_0^{\tex},
    x_0^{\sem}
  \big]\in \mathbb{R}^{d_{\tex}+d_{\sem}},
\end{equation}
where \(x_1\) is the encoded data latent and \(x_0\sim\mathcal{N}(0,I)\) is Gaussian noise with the
same shape, drawn independently of \(x_1\) and the class label \(c\). We flatten spatial dimensions for notation, so
\(x_1^{\tex},x_0^{\tex}\in\mathbb{R}^{d_{\tex}}\) and
\(x_1^{\sem},x_0^{\sem}\in\mathbb{R}^{d_{\sem}}\).

Let $\mathcal{G}=\{\tex,\sem\}$ denote the set of representation groups. For each group \(g\in\mathcal{G}\), let \(t_g\in[0,1]\) denote
the \emph{local time} for that group. We apply the standard linear flow-matching
path (Eq.~\eqref{eq:standard-fm-path} in Appendix~\ref{app:fm-background})
separately in each representation space; in our convention, \(t_g=0\)
corresponds to $\mathcal{N}(0,I)$ noise and \(t_g=1\) corresponds to clean data:
\begin{equation}
  \label{eq:component-interpolant}
  x^g(t_g)
  =
  (1-t_g)x_0^g + t_g x_1^g,
  \qquad
  u^g(t_g)
  =
  \frac{d x^g(t_g)}{d t_g}
  =
  x_1^g - x_0^g.
\end{equation}
The velocity \(u^g(t_g)\) is a sample-wise local-time
velocity: it depends only on the endpoint pair
\((x_0^g,x_1^g)\) in group \(g\), and measures motion per unit change in that
group's local time \(t_g\). For the linear path above, this velocity is
constant in \(t_g\), but we keep the time argument to emphasize which local
time it is associated with.

For fixed local times \((t_{\tex},t_{\sem})\), define the asynchronous
composite state
\(
  z(t_{\tex},t_{\sem})
  =
  [
    x^{\tex}(t_{\tex}),
    x^{\sem}(t_{\sem})
  ],
\)
and the corresponding asynchronous noised distribution
\(
  p_{t_{\tex},t_{\sem}}(\cdot\mid c)
  =
  \mathrm{Law}
  (
    z(t_{\tex},t_{\sem})
    \mid c
  ).
\)
The ideal local flow for group \(g\) is the conditional mean of the local
velocity:
\begin{equation}
  \label{eq:ideal-local-flow}
  v_g^\star(z,t_{\tex},t_{\sem},c)
  =
  \mathbb{E}
  \left[
    u^g(t_g)
    \mid
    z(t_{\tex},t_{\sem})=z,\; c
  \right],
  \qquad
  g\in\mathcal{G},
\end{equation}
where \(t_g=t_{\tex}\) for \(g=\tex\) and \(t_g=t_{\sem}\) for
\(g=\sem\). 

\paragraph{Unifying component flows under a single global time.}
An asynchronous denoising trajectory allows the two groups to occupy different
local times. However, the actual sampler still performs \textbf{one sequence of model
evaluations}, so we need a single global time \(\tau\in[0,1]\) indexing the
steps of the denoising algorithm. A schedule specifies how each local time
moves as a function of this global time:
\(
  t_{\tex}(\tau),
  \
  t_{\sem}(\tau),
  \
  \tau\in[0,1],
\)
with endpoint constraints
\(
  t_{\tex}(0)=t_{\sem}(0)=0,
  \
  t_{\tex}(1)=t_{\sem}(1)=1.
\)
The scheduled asynchronous state and its distribution are
\begin{equation}
  \label{eq:scheduled-async-state-general}
  z(\tau)
  =
  z(t_{\tex}(\tau),t_{\sem}(\tau))
  =
  \big[
    x^{\tex}(t_{\tex}(\tau)),
    x^{\sem}(t_{\sem}(\tau))
  \big],
\end{equation}
and
\(
  p_\tau(\cdot\mid c)
  =
  p_{t_{\tex}(\tau),t_{\sem}(\tau)}(\cdot\mid c)
  =
  \mathrm{Law}
  (
    z(\tau)
    \mid c
  ).
\)
Differentiating the scheduled state gives
\begin{equation}
  \label{eq:general-async-velocity}
  \dot z(\tau)
  =
  \frac{d z(\tau)}{d\tau}
  =
  \left[
    t_{\tex}'(\tau) u^{\tex}(t_{\tex}(\tau)),
    \;
    t_{\sem}'(\tau) u^{\sem}(t_{\sem}(\tau))
  \right].
\end{equation}
The schedule
derivatives in Eq.~\eqref{eq:general-async-velocity} convert local-time
velocities into the global-time velocity of the full state used by the sampler.

The next theorem formalizes these relationships.

\begin{theorem}[Ideal asynchronous flow]
\label{thm:ideal-async-flow}
Let \(p_1(\cdot\mid c)\) denote the distribution of clean composite latents
\(x_1=[x_1^{\tex},x_1^{\sem}]\) conditioned on class \(c\), and let
\(p_0=\mathcal{N}(0,I_{d_{\tex}+d_{\sem}})\) denote the base Gaussian
distribution, with \(x_0\sim p_0\) independent of \((x_1,c)\) and
\(\mathbb{E}\|x_1\|^2<\infty\). Let \(t_{\tex}(\tau)\) and \(t_{\sem}(\tau)\) be
continuously differentiable, nondecreasing local-time schedules with
\(t_g(\tau)<1\) for \(\tau<1\), satisfying
\(
  t_{\tex}(0)=t_{\sem}(0)=0,
  \
  t_{\tex}(1)=t_{\sem}(1)=1.
\)
Then the following hold:

\smallskip
\noindent\textbf{(I) Global-time flow.}
The global-time vector field
\begin{equation}
  \label{eq:global-target-field}
  V^\star(z,\tau,c)
  =
  \left[
    t_{\tex}'(\tau)
    v_{\tex}^\star
    \big(
      z,
      t_{\tex}(\tau),
      t_{\sem}(\tau),
      c
    \big),
    \;
    t_{\sem}'(\tau)
    v_{\sem}^\star
    \big(
      z,
      t_{\tex}(\tau),
      t_{\sem}(\tau),
      c
    \big)
  \right]
\end{equation}
drives the scheduled distributional path \(p_\tau\). Formally,
\(V^\star\) and \(p_\tau\) satisfy the continuity equation
$
  \partial_\tau p_\tau(z\mid c)
  +
  \nabla_z\!\cdot
  \left(
    p_\tau(z\mid c)
    V^\star(z,\tau,c)
  \right)
  =
  0.
$
Consequently, under the usual regularity conditions for the probability-flow
ODE, the solution of
$
  \frac{dZ_\tau}{d\tau}
  =
  V^\star(Z_\tau,\tau,c)
$
has marginals \(p_\tau(\cdot\mid c)\) and transports Gaussian noise at
\(\tau=0\) to the data-latent distribution at \(\tau=1\).

\smallskip
\noindent\textbf{(II) Score characterization.}
For the Gaussian linear noising path in Eq.~\eqref{eq:component-interpolant},
the ideal flow of each group \(g\in\{\tex,\sem\}\) satisfies
\begin{equation}
  \label{eq:score-transform-tex}
  v_{g}^\star(z,t_{\tex},t_{\sem},c)
  =
  \frac{1}{t_{g}}z^{g}
  +
  \frac{1-t_{g}}{t_{g}}\,
  \nabla_{z^{g}}
  \log p_{t_{\tex},t_{\sem}}(z\mid c),
  \qquad
  0<t_{\tex},t_{\sem}<1.
\end{equation}
The explicit form of \(p_\tau\) and the resulting score form of \(V^\star\) are
given with the proof in Appendix~\ref{ss:proof:thm:ideal-async-flow}.
\end{theorem}

\subsection{Flow Matching under a Fixed Asynchronous Schedule}
\label{sec:method:afm}

\paragraph{A simplifying assumption.} For the ideal flow, only the relative speed of the two local-time schedules matters. Assuming the semantic local-time schedule
is strictly increasing, we can reparameterize global time so that semantic time itself is the
global time:
\(
  t_{\sem}(\tau)=\tau.
\)
We write the texture time as \(t_{\tex}(\tau;\rho)\), where \(\rho\)
parameterizes the schedule. The scheduled asynchronous state and its
global-time velocity are
\begin{align}
  \label{eq:scheduled-async-state-velocity}
  z_\rho(\tau)
  &=
  \big[
    x^{\tex}(t_{\tex}(\tau;\rho)),
    x^{\sem}(\tau)
  \big],
  \qquad
  \dot z_\rho(\tau)
  =
  \left[
    t_{\tex}'(\tau;\rho) u^{\tex}(t_{\tex}(\tau;\rho)),
    \;
    u^{\sem}(\tau)
  \right].
\end{align}
Since larger local time means less noise, the constraint
\(
  t_{\tex}(\tau;\rho)\le \tau
  \
  \forall \tau\in[0,1]
\)
makes the semantic component cleaner than the texture component along the
trajectory. We refer to this as a semantic-leading schedule.

\paragraph{Component flow-matching loss.}
The ideal local flows in Eq.~\eqref{eq:ideal-local-flow} are not available in
closed form, so we learn them by flow matching. For a texture schedule
\(t_{\tex}(\tau;\rho)\), the scheduled state \(z_\rho(\tau)\) has distribution
\(
  p_\tau^\rho(\cdot\mid c)
  =
  p_{t_{\tex}(\tau;\rho),\tau}(\cdot\mid c)
  =
  \mathrm{Law}
  (
    z_\rho(\tau)
    \mid c
  ),
\)
the scheduled version of \(p_{t_{\tex},t_{\sem}}\) from
Section~\ref{sec:method:prelim}.

A $\theta$-parameterized denoising network receives the asynchronous state, the two local times, and the class condition, and predicts group-wise local velocities $\hat v_\theta^{\tex}(z,t_{\tex},t_{\sem},c)$ and $\hat v_\theta^{\sem}(z,t_{\tex},t_{\sem},c)$.

For $g\in \{\tex,\sem\}$, define the component-wise norm $\|a^g\|_g^2 := \frac{1}{d_g}\|a^g\|_2^2$. We define the \textbf{component flow-matching losses} as
\begin{equation}
  \label{eq:component-loss-tex}
  \ell_{g}(\theta;t_{\tex},t_{\sem})
  =
  \mathbb{E}_{x_0,x_1,c}
  \left[
    \left\|
      \hat v_\theta^{g}
      \big(
        z(t_{\tex},t_{\sem}),
        t_{\tex},
        t_{\sem},
        c
      \big)
      -
      u^{g}(t_{g})
    \right\|_{g}^2
  \right],
  \qquad
  g\in\{\tex,\sem\}.
\end{equation}
The targets $u^g$, as defined in ~\eqref{eq:component-interpolant}, are \emph{velocities with respect to local time $t_g$}.
For any fixed parametric schedule \(t_{\tex}(\tau;\rho)\), the \textbf{global flow-matching objective} is the sum of the texture and semantic component losses:
\begin{equation}
  \label{eq:generic-async-fm-loss}
  \mathcal{L}_{\mathrm{flow}}(\theta,\rho;\omega)
  =
  \mathbb{E}_{\tau\sim\mathcal{U}(0,1)}
  \left[
    \omega_{\tex}(\tau,\rho)
    \ell_{\tex}
    \big(
      \theta;
      t_{\tex}(\tau;\rho),
      \tau
    \big)
    +
    \omega_{\sem}(\tau,\rho)
    \ell_{\sem}
    \big(
      \theta;
      t_{\tex}(\tau;\rho),
      \tau
    \big)
  \right],
\end{equation}
where \(\omega_{\tex}(\tau,\rho)>0\) and
\(\omega_{\sem}(\tau,\rho)>0\) are generic time weights. These weights determine
how different portions of the trajectory are emphasized during training. We
discuss the proper choice of \(\omega_{\tex}\) and \(\omega_{\sem}\) in
Section~\ref{sec:method:objective}; for now, we treat them as arbitrary positive
weight functions. 

Theorem~\ref{thm:async-fm-population-optimum} in
Appendix~\ref{ss:proof:thm:async-fm-population-optimum} states the population target learned by
Eq.~\eqref{eq:generic-async-fm-loss}. For fixed \(\rho\), minimizing
\(\mathcal{L}_{\mathrm{flow}}(\theta,\rho;\omega)\) over \(\theta\) recovers
the ideal local flows defined in \eqref{eq:ideal-local-flow} along the schedule.

\paragraph{Inference.}
We freeze the learned schedule $\rho$, and write
\(
  t_{\tex}^\star(\tau)
  =
  t_{\tex}(\tau;\rho^\star),
\)
written \(f^\star(\tau)\) outside this section. We sample \(z_0=[z_0^{\tex},z_0^{\sem}]\) from Gaussian noise and evolve a
single global-time trajectory from \(\tau=0\) to \(\tau=1\). Since the
denoiser predicts local velocities, the texture branch is converted to
global time by the chain rule:
\begin{equation}
  \label{eq:inference-ode}
  \frac{d z_\tau^{\tex}}{d\tau}
  =
  {t_{\tex}^\star}'(\tau)\,
  \hat v_\theta^{\tex}
  \big(
    z_\tau,
    t_{\tex}^\star(\tau),
    \tau,
    c
  \big),
  \qquad
  \frac{d z_\tau^{\sem}}{d\tau}
  =
  \hat v_\theta^{\sem}
  \big(
    z_\tau,
    t_{\tex}^\star(\tau),
    \tau,
    c
  \big).
\end{equation}
Because \(t_{\tex}^\star(0)=0\), \(t_{\tex}^\star(1)=1\), and
\(t_{\sem}(\tau)=\tau\), the local times of both representation groups run from
noise (\(t=0\)) to clean data (\(t=1\)). The final image is decoded from the texture latent
\(z_1^{\tex}\); the semantic latent is generated jointly and used only as an
auxiliary representation along the denoising trajectory.

\subsection{Objective for Learning the Schedule}
\label{sec:method:objective}

Section~\ref{sec:method:afm} defines the weighted flow objective
\(\mathcal{L}_{\mathrm{flow}}(\theta,\rho;\omega)\) for an arbitrary choice of
time weights \(\omega=(\omega_{\tex},\omega_{\sem})\). We now choose these
weights so that the same objective is useful for learning the schedule
\(t_{\tex}(\tau;\rho)\).

\paragraph{Invariant local-time weighting criteria.}
The schedule should decide which semantic time is paired with each texture time,
but it should not change how much training weight each local noise level
receives. Since we set \(t_{\sem}(\tau)=\tau\), a schedule pairs semantic local
time \(s\) with texture local time \(t_{\tex}(s;\rho)\). Conversely, a texture
local time \(s\) is paired with semantic time \(t_{\tex}^{-1}(s;\rho)\), where
\(t_{\tex}^{-1}(\cdot;\rho)\) is the inverse of \(t_{\tex}(\cdot;\rho)\) in \(\tau\).

When learning the $\rho$-parameterized schedule, we want the weighted flow objective to have the following
local-time form \emph{for all $\rho$}:
\begin{equation}
  \label{eq:local-time-invariance-condition}
  \begin{aligned}
  \mathcal{L}_{\mathrm{flow}}(\theta,\rho;\omega)
  =
  {}&
  \mathbb{E}_{s_{\tex}\sim\mathcal{U}(0,1)}
  \left[
    \ell_{\tex}
    \big(
      \theta;
      s_{\tex},
      t_{\tex}^{-1}(s_{\tex};\rho)
    \big)
  \right]
  \\
  &+
  \mathbb{E}_{s_{\sem}\sim\mathcal{U}(0,1)}
  \left[
    \ell_{\sem}
    \big(
      \theta;
      t_{\tex}(s_{\sem};\rho),
      s_{\sem}
    \big)
  \right].
  \end{aligned}
\end{equation}
This condition keeps the marginal weighting over texture local time uniform and
the marginal weighting over semantic local time uniform. Changing \(\rho\) still
changes the cross-representation pairing inside the loss arguments, but it does
not change how much total loss weight is assigned to each local time.

The subtlety is that minibatch training samples the global time \(\tau\), not
the local texture time. Therefore, a fixed global-time weight need not
correspond to a fixed local-time weight. For example, if one uses the
uncorrected texture weight \(\omega_{\tex}(\tau,\rho)=1\), then the effective
distribution on local texture time $s_{\tex}$ is no longer uniform, but is instead
\(
  1/
    t_{\tex}'
    (
      t_{\tex}^{-1}(s_{\tex};\rho);
      \rho
    ).
\)
Thus the schedule can change the objective in two ways at once: it changes the
semantic--texture ordering, and it changes the amount of training weight
assigned to different texture noise levels. The second effect is a confound. In
particular, the optimizer could reduce the contribution of a difficult
texture-time region by making \(t_{\tex}(\tau;\rho)\) pass through that region
quickly, rather than by finding a better denoising order.

The following lemma states a \emph{necessary and sufficient} $\omega$ choice that removes this confound.

\begin{lemma}[Local-time invariant weighting]
\label{lem:schedule-time-reweighting}
Assume \(t_{\tex}(\cdot;\rho)\) is continuously differentiable with
\(t_{\tex}(0;\rho)=0\), \(t_{\tex}(1;\rho)=1\), and \(t_{\tex}'(\tau;\rho)>0\) for
\(\tau\in(0,1)\). As an identity
for arbitrary component losses, the local-time invariance condition in
Eq.~\eqref{eq:local-time-invariance-condition} is satisfied if and only if
\begin{equation}
  \label{eq:invariant-weights}
  \omega_{\tex}(\tau,\rho)
  =
  t_{\tex}'(\tau;\rho),
  \qquad
  \omega_{\sem}(\tau,\rho)
  =
  1
\end{equation}
almost everywhere.
\end{lemma}
We defer the proof of Lemma \ref{lem:schedule-time-reweighting} to Appendix \ref{ss:proof:lem:schedule-time-reweighting}.

Finally, we add fixed group weights \(w_{\tex}\) and \(w_{\sem}\) to balance the
two representation groups; as constants, they rescale each group's local-time
weighting without changing its shape. The corrected weights used by our method are
\begin{equation}
  \label{eq:corrected-weights}
  \omega_{\tex}^{\mathrm{corr}}(\tau,\rho)
  =
  w_{\tex}\,
  \operatorname{sg}
  \left(
    t_{\tex}'(\tau;\rho)
  \right),
  \qquad
  \omega_{\sem}^{\mathrm{corr}}(\tau,\rho)
  =
  w_{\sem},
\end{equation}
where \(\operatorname{sg}(\cdot)\) denotes stop-gradient. The forward value of
\(t_{\tex}'(\tau;\rho)\) performs the change-of-variables correction. We stop
gradients through this factor because it is an importance-weighting correction,
not a schedule objective: the schedule update treats it as a fixed weight and
differentiates only through the texture time at which the loss is evaluated.

We plug \eqref{eq:corrected-weights} into \eqref{eq:generic-async-fm-loss} and define (with slight abuse of notation) the corrected flow objective as
\begin{equation}
  \label{eq:flow-objective}
  \mathcal{L}_{\mathrm{flow}}(\theta,\rho)
  :=
  \mathcal{L}_{\mathrm{flow}}
  \big(
    \theta,
    \rho;
    \omega^{\mathrm{corr}}
  \big).
\end{equation}
The explicit global-time and local-time forms of Eq.~\eqref{eq:flow-objective} are given in Appendix~\ref{ss:corrected-forms}.
Thus, for every schedule, the corrected objective keeps the marginal local-time
weighting fixed, and its value depends on \(\rho\) only through the
cross-representation ordering.

\paragraph{Kinetic regularization.}
The corrected flow objective controls the training-time velocity regression
problem, but it does not by itself control the geometry of the sampling
trajectory. A schedule can fit local velocities well while still being poor for
finite-step ODE sampling; for example, it may compress most texture denoising
into a short interval of global time.

Motivated by the global-time flow in Eq.~\eqref{eq:global-target-field},
we penalize the squared speed of the predicted global-time velocity along the
training interpolation:
\begin{equation}
  \label{eq:reg}
  \begin{aligned}
  \mathcal{R}_{\mathrm{kin}}(\theta,\rho)
  =
  \mathbb{E}_{\tau\sim\mathcal{U}(0,1),x_0,x_1,c}
  \bigg[
  {}&
    t_{\tex}'(\tau;\rho)^2
    \left\|
      \hat v_\theta^{\tex}
      \big(
        z_\rho(\tau),
        t_{\tex}(\tau;\rho),
        \tau,
        c
      \big)
    \right\|_{\tex}^2
  \\
  &+
    \left\|
      \hat v_\theta^{\sem}
      \big(
        z_\rho(\tau),
        t_{\tex}(\tau;\rho),
        \tau,
        c
      \big)
    \right\|_{\sem}^2
  \bigg].
  \end{aligned}
\end{equation}
The factor \(t_{\tex}'(\tau;\rho)^2\) appears because texture velocities are
predicted per unit local texture time but are used per unit global sampling
time. Penalizing this quantity discourages schedules that move texture too
quickly over a small number of sampling steps.

\paragraph{Schedule and denoiser objectives.}
The two objectives used in the remainder of the method are
\begin{align}
  \label{eq:rho-objective}
  \mathcal{L}_{\rho}(\theta,\rho)
  &=
  \mathcal{L}_{\mathrm{flow}}(\theta,\rho)
  +
  \lambda
  \mathcal{R}_{\mathrm{kin}}(\theta,\rho),
  \\
  \label{eq:theta-objective}
  \mathcal{L}_{\theta}(\theta,\rho)
  &=
  \mathcal{L}_{\mathrm{flow}}(\theta,\rho)
  +
  \mathcal{L}_{\mathrm{aux}}(\theta,\rho).
\end{align}
The schedule objective \(\mathcal{L}_{\rho}\) is used during the probe stage to
select a schedule that preserves local-time weighting and avoids texture motion
concentrated in a few sampling steps. The denoiser objective
\(\mathcal{L}_{\theta}\) is used after the schedule is fixed; its auxiliary term
\(\mathcal{L}_{\mathrm{aux}}\) keeps the remaining SFD~\citep{sfd} training recipe
unchanged, including the LightningDiT auxiliary losses and the REPA alignment
loss~\citep{repa}. Section~\ref{sec:method:probe} describes the two-stage
optimization procedure.

\subsection{Semantic-Leading Schedule Parameterization}
\label{sec:method:parameterization}

After setting semantic time as the global time,
\(t_{\sem}(\tau)=\tau\), the only schedule to learn is the texture schedule
\(t_{\tex}(\tau;\rho)\).

\paragraph{Convex monotone schedule family.}
We parameterize the derivative of the texture schedule as a normalized
polynomial with positive coefficients:
\begin{equation}
  \label{eq:fprime}
  t_{\tex}'(\tau;\rho)
  =
  \frac{1}{Z_\rho}
  \sum_{m=0}^{M} a_m \tau^m,
  \qquad
  a_m = \mathrm{softplus}(\rho_m) > 0,
  \qquad
  Z_\rho
  =
  \sum_{m=0}^{M}\frac{a_m}{m+1}.
\end{equation}
The texture schedule is obtained by closed-form integration:
\begin{equation}
  \label{eq:fintegral}
  t_{\tex}(\tau;\rho)
  =
  \int_0^\tau t_{\tex}'(s;\rho)\,ds
  =
  \frac{1}{Z_\rho}
  \sum_{m=0}^{M}
  \frac{a_m}{m+1}\tau^{m+1}.
\end{equation}
The normalization \(Z_\rho\) enforces \(t_{\tex}(1;\rho)=1\), while the
integral form gives \(t_{\tex}(0;\rho)=0\). Since the coefficients are
positive, \(t_{\tex}'(\tau;\rho)>0\) on \([0,1]\), so
\(t_{\tex}(\cdot;\rho)\) is strictly increasing and has a well-defined inverse.
This inverse is used in the local-time change-of-variables analysis in
Section~\ref{sec:method:objective}; the implementation itself samples
\(\tau\) and evaluates \(t_{\tex}(\tau;\rho)\) and
\(t_{\tex}'(\tau;\rho)\) directly.

\paragraph{Semantic-leading property.}
The same parameterization enforces semantic-leading schedules by convexity:
\(t_{\tex}(\tau;\rho)\le \tau\) for all \(\tau\in[0,1]\)
(Appendix~\ref{app:semantic-leading}).
Because larger local time means less noise, this inequality ensures that the
semantic component is always at least as clean as the texture component along
the trajectory. Thus the semantic-leading constraint is enforced by
construction, without an auxiliary penalty.

We use \(M=4\) in all experiments. The schedule parameters are initialized so
that \(t_{\tex}(\tau;\rho)\) is close to the identity schedule
\(t_{\tex}(\tau;\rho)\approx\tau\), and the schedule probe in
Section~\ref{sec:method:probe} then learns the semantic-leading deviation from
this initialization.

\subsection{Bilevel Schedule Probe}
\label{sec:method:probe}

The schedule objective in Eq.~\eqref{eq:rho-objective} defines what makes a
schedule useful, but it cannot be optimized over \(\rho\) \emph{in isolation}. The
quality of a schedule depends on the denoiser obtained after adapting to that
schedule, so schedule learning is naturally a bilevel problem
(Eq.~\eqref{eq:bilevel-schedule-learning} in Appendix~\ref{app:bilevel}).

This bilevel view has an important practical consequence: schedule gradients are
only meaningful when the denoiser tracks the schedule. If \(\theta\) is frozen,
changing \(\rho\) changes the asynchronous state \(z_\rho(\tau)\) and the time
inputs \((t_{\tex}(\tau;\rho),\tau)\) on which the denoiser is evaluated. The
resulting flow-matching error reflects a model--schedule mismatch, rather than
the intrinsic quality of the new schedule. In practice, this mismatch makes the
loss rise sharply under schedule perturbations, so \(\rho\) barely moves. We
therefore learn \(\rho\) using a temporary denoiser that is optimized jointly
with the schedule.

We do not solve Eq.~\eqref{eq:bilevel-schedule-learning} exactly. Full denoiser
training is expensive, but the schedule parameters are low-dimensional and
converge quickly. A short joint optimization of the denoiser and the schedule
is sufficient to identify a good schedule. We therefore use a two-stage recipe. In Stage I, we run a short joint
probe over a temporary denoiser \(\theta_{\mathrm{probe}}\) and the schedule
parameters \(\rho\), using only about \(1\%\) of the main training budget. In
Stage II, we freeze the learned schedule, discard
\(\theta_{\mathrm{probe}}\), and train the final denoiser \(\theta\) from
scratch. No schedule parameters are updated during the main training run.

During the probe, we optimize
\(
  \mathcal{L}_{\mathrm{probe}}
  (
    \theta_{\mathrm{probe}},
    \rho
  )
  =
  \mathcal{L}_{\rho}
  (
    \theta_{\mathrm{probe}},
    \rho
  ).
\)
Equivalently, using Eq.~\eqref{eq:rho-objective}, this is
\(
  \mathcal{L}_{\mathrm{flow}}
  (
    \theta_{\mathrm{probe}},
    \rho
  )
  +
  \lambda
  \mathcal{R}_{\mathrm{kin}}
  (
    \theta_{\mathrm{probe}},
    \rho
  ).
\)
The probe denoiser is not intended to be a final generative model; it only
tracks the changing flow-matching problem well enough to supply useful gradients
for the schedule. We choose \(\lambda\) from a one-dimensional sweep as the
weakest value that prevents collapse toward the extremal schedule. This
transition is visualized in Section~\ref{sec:exp:schedule}.

After the probe, we average the schedule parameters over the stable
post-burn-in window and freeze the resulting schedule
\(t_{\tex}^{\star}(\tau)\). The final denoiser is then trained from scratch with
this fixed schedule using \(\mathcal{L}_{\theta}\), including the unchanged SFD
auxiliary losses.

Algorithm~\ref{alg:probe} in Appendix~\ref{app:algorithm} summarizes the procedure.

\section{Experiments}
\label{sec:exp}

\subsection{Experimental Setup}
\label{sec:exp:setup}

\paragraph{Task and protocol.}
We evaluate on class-conditional ImageNet-\(256\times256\) generation. Our goal
is to isolate the effect of schedule discovery. Therefore, all controlled
comparisons use the SFD framework~\citep{sfd}. We keep the architecture, latent
representation, auxiliary losses, weak-model architecture for guidance, sampler, and
evaluation protocol unchanged; our changes are confined to the learned fixed
schedule \(f^\star\) and the schedule-learning procedure used to obtain it.

\paragraph{Baselines.}
Our most direct baseline is SFD~\citep{sfd}, which shares with us the
asynchronous semantic--texture latent representation, the LightningDiT-XL/1
backbone, the weak-model architecture and training budget for AutoGuidance, and the sampler; we report the
released SFD-XL and SFD-XXL checkpoints. For broader context in
Table~\ref{tab:fid-ag}, we also include representative class-conditional
ImageNet $256\times256$ generators (listed in Appendix~\ref{app:exp_details}). These broader
comparisons are not fully controlled, but contextualize the system-level
performance of our learned schedule.

\paragraph{Metrics.}
Following standard ImageNet generation practice, we generate 50K
class-balanced samples and report FID~\citep{fid}, sFID~\citep{sfid}, Inception
Score (IS)~\citep{inceptionscore}, Precision, and Recall~\citep{precisionrecall}.
Unless otherwise stated, samples are
evaluated against the ADM reference statistics~\citep{adm}.

\paragraph{Implementation.}
We train on the SFD-released ImageNet latent dataset with a
LightningDiT-XL/1 backbone (675M parameters) and batch size 256. The
schedule probe (Section~\ref{sec:method:probe}) is run for
\(S_{\mathrm{probe}}=10\mathrm{K}\) steps with \(S_{\mathrm{burn}}=5\mathrm{K}\)
burn-in, approximately \(1\%\) of our \(1\mathrm{M}\)-iteration main
training runs; schedule discovery therefore adds negligible overhead on top of
main training. We report results across training budgets up to 3M iterations. We index
unguided convergence (Table~\ref{tab:fid-noguide}) by iteration count, following
the LightningDiT and SFD baselines, and the system-level comparison
(Table~\ref{tab:fid-ag}) by epoch, following the methods we compare against; at
batch size 256 on ImageNet-1k one epoch is roughly 5K iterations, so 80, 200,
and 600 epochs correspond to about 0.4M, 1M, and 3M iterations. Full training
and hardware details are given in Appendix~\ref{app:exp_details}.
Sampling integrates the ODE with the adaptive dopri5 solver on a grid of 250 time points,
and AutoGuidance~\citep{ag} uses a weak model with the same configuration as
SFD-XL's (LightningDiT-B trained for 70K steps), but trained under our learned
schedule, so that the weak and main models share the same asynchronous
denoising path. Full hyperparameter
settings are listed in Tables~\ref{tab:hparams-ours} and~\ref{tab:hparams-sfd}
of Appendix~\ref{app:exp_details}.

\subsection{Schedule Probe Diagnostics}
\label{sec:exp:schedule}

We sweep the kinetic-energy regularization strength
\(\lambda \in \{1, 2, 3, 4, 5, 6, 10\}\times10^{-2}\) during the probe and
observe two regimes (Figure~\ref{fig:lambda-sweep}). For
\(\lambda \le 3\times10^{-2}\) the regularizer is too weak to prevent
collapse, and the schedule converges to the same extremal semantic-leading
curve regardless of the exact value, so we plot a single representative
curve for this collapsed regime. For \(\lambda \ge 4\times10^{-2}\) the
schedule stabilizes into a smooth curve that becomes progressively closer
to the identity as \(\lambda\) grows. We use the weakest stable value,
\(\lambda = 4\times10^{-2}\), for all main experiments.

\subsection{Main Results}
\label{sec:exp:main}

We evaluate the learned schedule in two regimes. First, we compare unguided FID
across training budgets to measure convergence speed. Second, we evaluate the
final system with AutoGuidance and compare against state-of-the-art
class-conditional ImageNet generators. In both regimes, the most direct
comparison is SFD-XL, since Ours and SFD-XL use the same backbone, latent
representation, weak-model architecture, sampler, and evaluation protocol. They differ in
the fixed schedule used for asynchronous denoising and in the associated
corrected flow loss from Eq.~\eqref{eq:flow-objective}.

\begin{table}[t]
\vspace{-0.6em}
\centering

\begin{minipage}[t]{0.38\textwidth}
\vspace{0pt}
\centering

\begin{minipage}[t][5.5em][t]{\linewidth}
\centering
\caption{
\textbf{Unguided FID convergence.}
All entries use 675M-parameter XL backbones. Baseline numbers are from
\citet{dit,lightningdit,repa,sfd}.
}
\label{tab:fid-noguide}
\end{minipage}
\vspace{0.4em}

\scriptsize
\setlength{\tabcolsep}{2.0pt}
\renewcommand{\arraystretch}{1.00}
\begin{tabular*}{\linewidth}{@{\extracolsep{\fill}}lcc@{}}
\toprule
Model & Iter. & FID$\downarrow$ \\
\midrule
DiT-XL/2        & 400K & 19.47 \\
DiT-XL/2        & 7M   & 9.62 \\
LightningDiT    & 400K & 9.29 \\
LightningDiT    & 1M   & 7.48 \\
LightningDiT    & 2M   & 6.88 \\
LightningDiT    & 4M   & 6.50 \\
\quad + REPA    & 400K & 6.94 \\
\quad + REPA    & 1M   & 6.17 \\
\quad + REPA    & 2M   & 5.87 \\
\quad + REPA    & 4M   & 5.84 \\
\midrule
\quad + SFD     & 70K  & 8.79 \\
\quad + SFD     & 120K & 6.22 \\
\quad + SFD     & 400K & 3.53 \\
\quad + SFD     & 1M   & 2.82 \\
\quad + SFD     & 2M   & 2.74 \\
\quad + SFD     & 4M   & 2.54 \\
\midrule
\quad + Ours & 70K  & \textbf{6.89} \\
\quad + Ours & 120K & \textbf{4.93} \\
\quad + Ours & 400K & \textbf{2.87} \\
\quad + Ours & 800K & \textbf{2.53} \\
\quad + Ours & 1M   & \textbf{2.37} \\
\quad + Ours & 2M   & \textbf{2.21} \\
\quad + Ours & 3M   & \textbf{2.14} \\
\bottomrule
\end{tabular*}
\end{minipage}
\hfill
\begin{minipage}[t]{0.59\textwidth}
\vspace{0pt}
\centering

\begin{minipage}[t][5.5em][t]{\linewidth}
\centering
\caption{
\textbf{System-level comparison with AutoGuidance.}
Class-conditional ImageNet 256$\times$256 results under AutoGuidance;
baseline numbers are from the corresponding papers.
}
\label{tab:fid-ag}
\end{minipage}
\vspace{0.15em}

\scriptsize
\setlength{\tabcolsep}{1.6pt}
\renewcommand{\arraystretch}{1.00}
\begin{tabular*}{\linewidth}{@{\extracolsep{\fill}}lccccccc@{}}
\toprule
Method & Ep. & Par. & FID$\downarrow$ & sFID$\downarrow$ & IS$\uparrow$ & Prec.$\uparrow$ & Rec.$\uparrow$ \\
\midrule
\multicolumn{8}{@{}l}{\textit{Latent Diffusion Models}} \\
DiT-XL        & 1400 & 675M & 2.27 & 4.60 & 278.2 & \textbf{0.83} & 0.57 \\
MaskDiT       & 1600 & 675M & 2.28 & 5.67 & 276.6 & 0.80 & 0.61 \\
SiT-XL        & 1400 & 675M & 2.06 & 4.50 & 270.3 & 0.82 & 0.59 \\
FasterDiT     & 400  & 675M & 2.03 & 4.63 & 264.0 & 0.81 & 0.60 \\
MDT           & 1300 & 675M & 1.79 & 4.57 & 283.0 & 0.81 & 0.61 \\
MDTv2         & 1080 & 675M & 1.58 & 4.52 & \textbf{314.7} & 0.79 & 0.65 \\
DDT           & 400  & 675M & 1.26 & --   & 310.6 & 0.79 & 0.65 \\
\midrule
\multicolumn{8}{@{}l}{\textit{Leveraging Visual Representations}} \\
VA-VAE        & 800 & 675M & 1.35 & 4.15 & 295.3 & 0.79 & 0.65 \\
REPA          & 800 & 675M & 1.42 & 4.70 & 305.7 & 0.80 & 0.65 \\
REPA-E        & 800 & 675M & 1.12 & 4.09 & 302.9 & 0.79 & 0.66 \\
ReDi          & 800 & 675M & 1.61 & 4.66 & 295.1 & 0.78 & 0.64 \\
REG           & 800 & 677M & 1.36 & 4.25 & 299.4 & 0.77 & 0.66 \\
RAE-DiT       & 800 & 676M & 1.41 & --   & 309.4 & 0.80 & 0.63 \\
RAE-DiTDH     & 800 & 839M & 1.13 & --   & 262.6 & 0.78 & 0.67 \\
SFD-XL        & 80  & 675M & 1.30 & 3.87 & 233.4 & 0.78 & 0.64 \\
SFD-XL        & 800 & 675M & 1.06 & 3.89 & 267.0 & 0.78 & 0.67 \\
SFD-XXL       & 80  & 1.0B & 1.19 & 4.00 & 240.4 & 0.78 & 0.65 \\
SFD-XXL       & 800 & 1.0B & 1.04 & \textbf{3.75} & 264.2 & 0.78 & 0.66 \\
Ours-XL & 80  & 675M & 1.14 & 3.79 & 248.4 & 0.78 & 0.71 \\
Ours-XL & 200 & 675M & 1.05 & 3.79 & 273.0 & 0.78 & \textbf{0.72} \\
Ours-XL & 600 & 675M & \textbf{1.02} & 3.78 & 270.8 & 0.78 & 0.66 \\
\bottomrule
\end{tabular*}
\end{minipage}

\vspace{-0.7em}
\end{table}
\begin{figure}[t]
  \centering
  \begin{minipage}[t]{0.28\linewidth}
    \vspace{0pt}
    \centering
    \includegraphics[width=\linewidth]{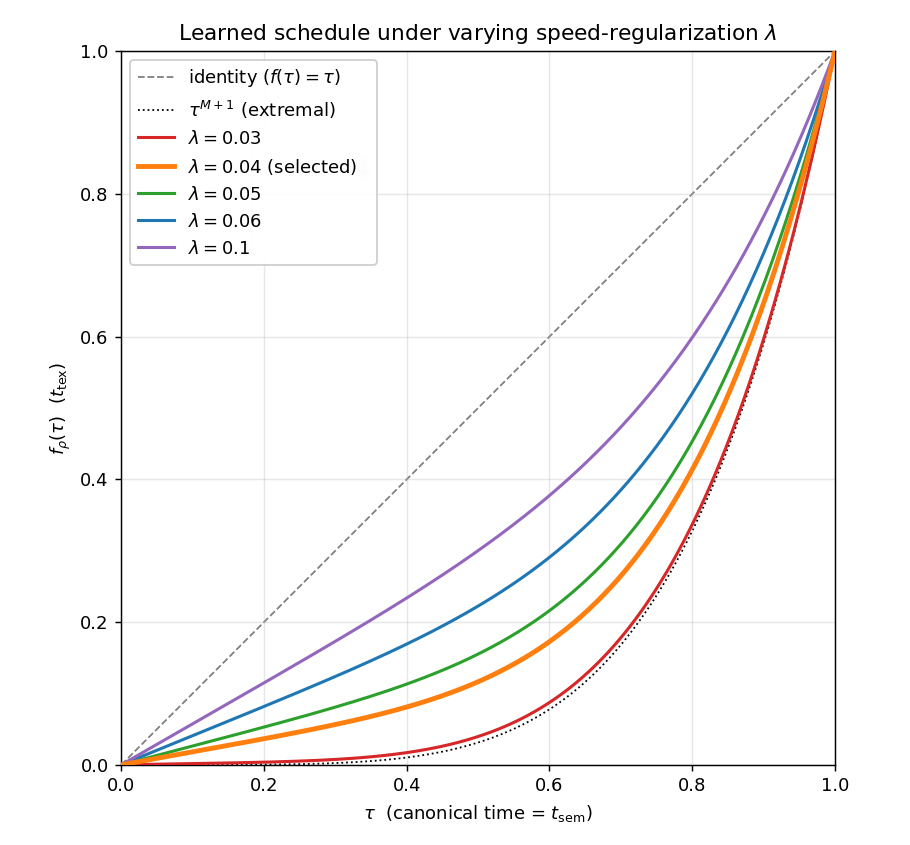}
    \caption{\(\lambda\)-sweep of the schedule probe.}
    \label{fig:lambda-sweep}
  \end{minipage}
  \hfill
  \begin{minipage}[t]{0.28\linewidth}
    \vspace{0pt}
    \centering
    \includegraphics[width=\linewidth]{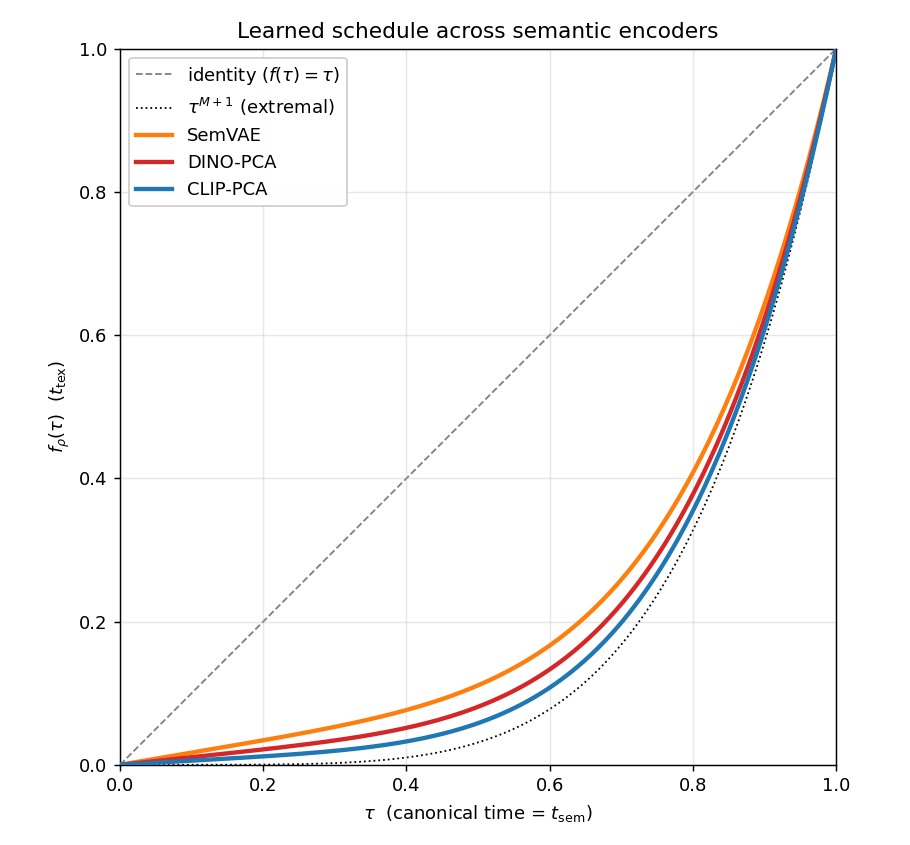}
    \caption{Learned schedules across semantic representations.}
    \label{fig:schedule-repr}
  \end{minipage}
  \hfill
  \begin{minipage}[t]{0.41\linewidth}
    \vspace{0pt}
    \centering
    \captionof{table}{
    \textbf{Robustness to semantic representations.}
    Unguided FID at 400K iterations; the SemVAE entry matches the 400K row of
    Table~\ref{tab:fid-noguide}. SFD numbers are from~\citet{sfd}.
    $^\dagger$SFD trains a CLIP-VAE on top of CLIP features; ours uses
    CLIP-PCA without that intermediate VAE.
    }
    \label{tab:repr-robustness}
    \scriptsize
    \setlength{\tabcolsep}{1.5pt}
    \renewcommand{\arraystretch}{1.05}
    \begin{tabular}{llcc}
    \toprule
    Feature source & Semantic latent & SFD~\citep{sfd} & Ours \\
    \midrule
    DINOv2-B~\citep{dinov2} & SemVAE                  & 3.53 & \textbf{2.87} \\
    DINOv2-B~\citep{dinov2} & PCA                     & 4.06 & \textbf{2.97} \\
    CLIP~\citep{clip}       & VAE / PCA$^{\dagger}$   & 4.89 & \textbf{4.54} \\
    \bottomrule
    \end{tabular}
  \end{minipage}
\end{figure}

\paragraph{Faster unguided convergence.}
Table~\ref{tab:fid-noguide} shows that the learned schedule improves
sample-efficiency over the fixed SFD schedule at every matched training budget.
The effect is largest in the low- and mid-compute regimes: at 400K iterations,
it improves FID from 3.53 to 2.87, and it reaches FID 2.53 at
800K iterations, matching the 4M-iteration SFD-XL checkpoint with roughly
\(5\times\) fewer diffusion-model updates.

\paragraph{Better final unguided quality.}
The gain is not only an early-training effect. At 1M iterations it
reaches FID 2.37, improving over the best reported SFD-XL unguided result of
2.54 at 4M iterations; extending training to 3M lowers the unguided FID
further to 2.14. Thus the learned schedule both accelerates convergence
and improves the final unguided model under the same architecture and latent
representation.

\paragraph{Both gains transfer under AutoGuidance.}
Table~\ref{tab:fid-ag} shows the same pattern under
AutoGuidance~\citep{ag}. With the same weak-model architecture and guidance
scale as SFD-XL, our model cuts SFD-XL's FID from 1.30 to
\textbf{1.14} at 80 epochs and reaches \textbf{FID 1.047} (1.05 in
Table~\ref{tab:fid-ag}) at 200 epochs, on par with the 1.0B-parameter
SFD-XXL (1.04) despite a smaller backbone and one quarter of the training
budget. Training longer to 600 epochs reaches
\textbf{FID 1.02}, the lowest among all 675M-parameter entries
in Table~\ref{tab:fid-ag} and below the 1.0B-parameter SFD-XXL. On Recall, our
200-epoch model reaches \textbf{0.72}, the highest among all entries in
Table~\ref{tab:fid-ag}, indicating that the gains include broader
coverage of the data distribution rather than sharper modes alone.
Figure~\ref{fig:qualitative} in Appendix~\ref{app:qualitative} shows class-conditional samples from the
same guided sampling setup.

\subsection{Robustness to Semantic Representations}
\label{sec:exp:repr-robustness}

\paragraph{Alternative semantic representations.}
The main experiments use the SFD SemVAE semantic latent, itself trained
as a VAE on top of DINOv2-B~\citep{dinov2} features. To test whether the
schedule-learning procedure depends on this encoder, we evaluate two
simpler alternatives that skip the SemVAE training step:
\emph{DINO-PCA}, a direct PCA projection of the same DINOv2-B feature
dataset that SemVAE was trained on, and \emph{CLIP-PCA}, a PCA
projection of CLIP~\citep{clip} features on ImageNet, both to the same
channel count as SemVAE. The only changes from the main setup are the
semantic latent and the per-encoder schedule. The kinetic-regularizer strength
is set per semantic representation: SemVAE uses the main value
\(\lambda=4\times10^{-2}\) (Section~\ref{sec:exp:schedule}), and DINO-PCA and
CLIP-PCA share \(\lambda=2\times10^{-2}\).

\paragraph{Schedule shape varies with the semantic encoder.}
The frozen schedules differ across encoders
(Figure~\ref{fig:schedule-repr}): CLIP-PCA produces the most aggressive
semantic-leading curve, DINO-PCA sits in the middle, and SemVAE stays
closest to the identity. Roughly, the richer the semantic latent's
information about the image, the less the texture branch is delayed.

\paragraph{Results.}
At 400K iterations (Table~\ref{tab:repr-robustness}), the learned
schedule reaches FID 2.97 with DINO-PCA and 4.54 with CLIP-PCA. Under
the same two encoders, \citet{sfd} report 4.06 with DINO-PCA and 4.89
with a CLIP-VAE trained on CLIP features. Schedule learning thus
improves over the hand-tuned offset at every encoder, and in the CLIP
case our PCA projection already beats SFD's CLIP-VAE recipe while
skipping the VAE-training step entirely.

\section{Conclusion}
\label{sec:conclusion}

We replace the hand-tuned semantic-leading offset of SFD with a schedule
\(f^\star\) learned as a convex monotone bijection of the global
time, so the leading property holds without an auxiliary constraint. A
short regularized probe discovers \(f^\star\) jointly with a temporary
denoiser, while a kinetic-energy regularizer keeps the resulting curve
away from the extremal collapse that fits the training loss but is too
stiff to integrate in finite steps. The averaged curve is frozen for
full training, leaving the final inference interface identical to a
fixed-schedule asynchronous model. On ImageNet-\(256\times256\), this
gives faster unguided convergence and better final FID than the
hand-tuned schedule under the same backbone, and under AutoGuidance reaches
FID \(1.02\), the lowest among the \(675\)M-parameter systems in our comparison
and below the \(1.0\)B-parameter SFD-XXL. The gain also transfers
across DINO- and CLIP-based semantic encoders. We discuss limitations
and directions for future work in Appendix~\ref{app:limitations}.

\begin{ack}
Part of the computation was performed on the Duke Compute Cluster. This work received no third-party funding. The
authors declare no competing interests.
\end{ack}

\bibliographystyle{plainnat}
\bibliography{ref}

\newpage

\appendix

\section{Limitations and Future Work}
\label{app:limitations}

Our experiments cover a single visual modality at a single resolution
and backbone size: class-conditional ImageNet-\(256\times256\) with the
LightningDiT-XL/1 backbone. While the gains transfer across the three
semantic encoders we evaluate
(Section~\ref{sec:exp:repr-robustness}), we do not study text-to-image,
video, or audio diffusion, and we do not scale beyond 675M parameters.
The interaction between \(f^\star\) and other guidance methods is also
limited in scope: we use AutoGuidance with the same weak-model configuration as
SFD-XL, and we have not separately characterized how
the learned schedule interacts with classifier-free guidance or with
guidance weights that themselves vary along \(\tau\).

We parameterize \(t_{\tex}'(\tau;\rho)\) as a degree-\(M=4\) polynomial. This is
the minimal expressive degree that gave a non-collapsed curve under the
kinetic regularizer in preliminary tests, and we have not compared
richer monotone families such as splines or input-convex networks. The
probe also requires a per-encoder sweep of \(\lambda\), since the value
that prevents collapse shifts with the semantic encoder, and
transferring the procedure to a new encoder involves a small
one-dimensional search.

A natural next direction is to extend the framework beyond two
representation groups, learning a joint schedule among a texture
latent, a semantic latent, and additional auxiliary representations
such as depth or segmentation. The convex-monotone parameterization
extends directly to a collection of pairwise schedules, but the
kinetic regularizer and probe budget would need to be revisited at
this larger scale.

\section{Additional Experimental Details}
\label{app:exp_details}

\paragraph{Baselines.}
The representative class-conditional ImageNet $256\times256$ generators in
Table~\ref{tab:fid-ag} fall into two groups. \emph{Latent diffusion
transformers}: DiT~\citep{dit}, SiT~\citep{sit}, MaskDiT~\citep{maskdit},
FasterDiT~\citep{fasterdit}, MDT~\citep{mdt}, MDTv2~\citep{mdtv2}, and
DDT~\citep{ddt}. \emph{Methods leveraging visual representations}:
VA-VAE~\citep{lightningdit}, REPA~\citep{repa}, REPA-E~\citep{repae},
ReDi~\citep{redi}, REG~\citep{reg}, and RAE~\citep{rae}.

\paragraph{Metrics.}
FID measures distributional visual quality, sFID emphasizes spatial statistics,
IS measures class-conditional sample quality and diversity, while Precision and
Recall separate fidelity from coverage.

\paragraph{Setup details.}
We use the SFD~\citep{sfd} implementation and its released
ImageNet-\(256\times256\) latent dataset. Texture latents come from the
SD-VAE f16-d32 of LightningDiT~\citep{lightningdit}; semantic latents
come from the SemVAE encoder, which compresses DINOv2-B~\citep{dinov2}
patch features with registers~\citep{registers} into a compact semantic
latent. The diffusion backbone is LightningDiT-XL/1, with
REPA~\citep{repa} alignment between an internal DiT block and DINOv2-B
final-layer features. The implementation averages the squared errors over all
48 channels and multiplies the 16 semantic channels by 2; since the texture
group has 32 channels, this equals \(w_{\tex}=w_{\sem}=2/3\) on the per-group
normalized losses of Eq.~\eqref{eq:component-loss-tex}, i.e., equal weights
for the two groups. FID is computed against the
ADM~\citep{adm} reference statistics. All remaining numerical settings
are listed in Tables~\ref{tab:hparams-ours} (schedule learning) and
\ref{tab:hparams-sfd} (inherited).

\begin{table}[!htbp]
\centering
\caption{Hyperparameters introduced or chosen in this work for schedule
learning. All other settings are inherited from
SFD~\citep{sfd}/LightningDiT~\citep{lightningdit} and listed in
Table~\ref{tab:hparams-sfd}.}
\label{tab:hparams-ours}
\small
\setlength{\tabcolsep}{6pt}
\renewcommand{\arraystretch}{1.10}
\begin{tabular}{ll}
\toprule
\textbf{Setting} & \textbf{Value} \\
\midrule
\multicolumn{2}{l}{\textit{Schedule parameterization}} \\
\midrule
Polynomial degree \(M\)             & 4 \\
\(\rho\) initialization             & near identity (\(t_{\tex}(\tau;\rho)\approx\tau\)) \\
Group weight \(w_{\tex}\)   & \(2/3\) \\
Group weight \(w_{\sem}\)   & \(2/3\) \\
Stop-grad on \(t_{\tex}'(\tau;\rho)\) factor     & yes \\
\midrule
\multicolumn{2}{l}{\textit{Schedule probe}} \\
\midrule
Probe steps \(S_{\mathrm{probe}}\)  & 10K \\
Burn-in steps \(S_{\mathrm{burn}}\) & 5K \\
Kinetic regularizer \(\lambda\)     & \(4\times10^{-2}\) \\
Probe optimizer                     & AdamW \\
Probe learning rate                 & \(10^{-4}\) (denoiser), \(10^{-2}\) (\(\rho\)) \\
Probe batch size                    & 256 \\
\bottomrule
\end{tabular}
\end{table}

\begin{table}[!htbp]
\centering
\caption{Hyperparameters inherited from SFD~\citep{sfd} and
LightningDiT~\citep{lightningdit}. Schedule-learning hyperparameters
introduced by this work are listed separately in
Table~\ref{tab:hparams-ours}.}
\label{tab:hparams-sfd}
\small
\setlength{\tabcolsep}{6pt}
\renewcommand{\arraystretch}{1.10}
\begin{tabular}{ll}
\toprule
\textbf{Setting} & \textbf{Value} \\
\midrule
\multicolumn{2}{l}{\textit{Architecture}} \\
\midrule
Backbone                            & LightningDiT-XL/1 \\
Parameters                          & 675M \\
DiT blocks                          & 28 \\
Hidden dim                          & 1152 \\
Attention heads                     & 16 \\
MLP ratio                           & 4.0 \\
Patch size                          & 1 \\
Texture latent shape                & \(32\times16\times16\) \\
Semantic latent shape               & \(16\times16\times16\) \\
SemVAE parameters                   & 29M \\
\midrule
\multicolumn{2}{l}{\textit{Main training}} \\
\midrule
Batch size                          & 256 \\
Optimizer                           & AdamW \\
Learning rate                       & \(10^{-4}\) \\
\(\beta_1\)                         & 0.9 \\
\(\beta_2\)                         & 0.999 \\
Weight decay                        & 0 \\
LR warmup                           & none \\
LR schedule                         & constant \\
Gradient clipping                   & none \\
EMA decay                           & 0.9999 \\
Precision                           & bf16 mixed \\
\midrule
\multicolumn{2}{l}{\textit{Auxiliary losses}} \\
\midrule
REPA alignment from / to            & DiT block 2 / DINOv2 final layer \\
REPA loss weight                    & 1.0 \\
Cosine-direction loss weight        & 1.0 \\
\midrule
\multicolumn{2}{l}{\textit{Sampling}} \\
\midrule
ODE solver                          & dopri5 \\
Output time points                  & 250 \\
\(\mathrm{atol}\)                   & \(10^{-6}\) \\
\(\mathrm{rtol}\)                   & \(10^{-3}\) \\
AutoGuidance scale \(w\)            & 1.5 (80/200 ep.), 1.3 (600 ep.) \\
Weak model                          & LightningDiT-B \\
Weak model training steps           & 70K \\
\midrule
\multicolumn{2}{l}{\textit{Evaluation}} \\
\midrule
Sample count                        & 50K, class-balanced \\
Reference batch                     & \texttt{VIRTUAL\_imagenet256\_labeled.npz} \\
Metrics                             & FID, sFID, IS, Precision, Recall \\
\bottomrule
\end{tabular}
\end{table}

\paragraph{Hardware and runtime.}
Training budgets up to 1M iterations were run on NVIDIA B200 and longer budgets
(2M and 3M iterations) on NVIDIA H200, under an identical recipe, sampler, and
evaluation protocol. On B200 the main training run uses 2 GPUs at a throughput
of 4.6 steps per second, giving approximately 60 hours of wall-clock time for
the 1M-iteration budget; the schedule probe runs for
$S_{\mathrm{probe}}=10\mathrm{K}$ steps and takes approximately 36 minutes,
consistent with the under-1\% overhead reported in Section~\ref{sec:exp:setup}.

\paragraph{Guidance scale.}
For AutoGuidance we use guidance scale \(w=1.5\) (matching SFD-XL) for the
80- and 200-epoch models, and \(w=1.3\) for the more-converged 600-epoch model.

\section{Additional Method Details}
\label{app:method}

\paragraph{Overview of Section~\ref{sec:method}.}
Our method learns an asynchronous denoising schedule for a composite latent
containing texture and semantic representations. Section~\ref{sec:method:prelim}
first defines the asynchronous flow induced by separate local-time schedules for the
two representation groups, and characterizes the corresponding ideal local and
global velocity fields. Section~\ref{sec:method:afm} then shows how these local
velocities are learned by flow matching for a fixed schedule, and how the learned
local velocities are converted into the global-time ODE used at inference.
Section~\ref{sec:method:objective}
defines the objective used to learn the schedule: a Jacobian-corrected flow loss
keeps the marginal weighting over local times invariant, while a kinetic
regularizer discourages schedules that are difficult to discretize. Finally,
Section~\ref{sec:method:parameterization} states our explicit parameterization of the learnable schedule: a construction which restricts the learnable texture schedule \(t_{\tex}(\tau;\rho)\) to a convex monotone family and enforces semantic-leading behavior. 
Section~\ref{sec:method:probe} describes the efficient two-stage optimization
procedure: a short joint probe learns the schedule, after which the schedule is
frozen and the final denoiser is trained from scratch.

\subsection{Standard flow matching}
\label{app:fm-background}

\paragraph{Background: standard flow matching.} We first recall the standard flow-matching construction. Let \(y_0\) denote a sample from a base distribution, let \(y_1\) denote a data sample given condition \(c\), drawn independently of \(y_0\), and define the linear path
\begin{equation}
  \label{eq:standard-fm-path}
  y(t)
  =
  (1-t)y_0 + t y_1,
  \qquad
  u
  =
  \frac{d y(t)}{dt}
  =
  y_1-y_0,
  \qquad
  t\in[0,1].
\end{equation}
In our convention, \(t=0\) corresponds to $\mathcal{N}(0,I)$ noise and \(t=1\) corresponds to clean data. Let
$
  p_t(\cdot\mid c)
  :=
  \mathrm{Law}(y(t)\mid c)
$ denote the interpolating distribution. The ideal flow field along this path is the conditional mean velocity
\begin{equation}
  \label{eq:standard-fm-ideal-flow}
  v^\star(y,t,c)
  =
  \mathbb{E}
  \left[
    u
    \mid
    y(t)=y,\; c
  \right].
\end{equation}
This field drives the curve of distributions \(\{p_t\}_{t\in[0,1]}\): under the usual regularity conditions, the ODE
\(dY_t/dt=v^\star(Y_t,t,c)\) transports the base distribution \(p_0\) to the
data distribution \(p_1\).

The component interpolant in Eq.~\eqref{eq:component-interpolant} has exactly the same form as
Eq.~\eqref{eq:standard-fm-path}, with \(y(t)=x^g(t_g)\), \(t=t_g\), and
\(u=u^g\).

Eq.~\eqref{eq:general-async-velocity} is the main difference from the standard path in
Eq.~\eqref{eq:standard-fm-path}. The standard interpolant \(y(t)\) has a single
time variable, so every coordinate is evaluated at the same noise level. The
asynchronous state \(z(\tau)\) has a single global sampler time \(\tau\),
but its texture and semantic components may be evaluated at different local
times \(t_{\tex}(\tau)\) and \(t_{\sem}(\tau)\); importantly, the noise levels
at \(z(\tau)\) differ in the \(\tex\) and \(\sem\) components.

\subsection{Velocity objects}

\paragraph{Our method centers around three subtly different velocity objects:}
\begin{enumerate}
    \item The endpoint velocity \(u^g(t_g)\) in \eqref{eq:component-interpolant} is a sample-wise velocity \textbf{determined by its own group's endpoint pair $(x_0^g,x_1^g)$}.
    \item The ideal local flow \(v_g^\star\) in \eqref{eq:ideal-local-flow} is a population velocity field obtained by \textbf{conditioning on the full composite state \(z\)}, so \(v_{\tex}^\star\) may depend on both the texture and semantic components, and likewise for \(v_{\sem}^\star\).
    \item The global flow \(V^\star\) in \eqref{eq:global-target-field}, is the global-time vector field obtained by concatenating the local flows and multiplying by the corresponding local-time schedule derivatives.
\end{enumerate}
Theorem~\ref{thm:ideal-async-flow} formalizes these relationships.

\subsection{Requirements for the schedule objective and parameterization}
\label{app:objective-requirements}

A good schedule objective should satisfy two
requirements. First, it should optimize the semantic--texture denoising order
\textbf{without changing the marginal weighting} over local noising times. Second, it
should prefer schedules whose global-time trajectories are \textbf{stable under
finite-step ODE sampling}. We address these requirements with a
change-of-variables reweighting and a kinetic regularizer
(Section~\ref{sec:method:objective}).

For the schedule itself (Section~\ref{sec:method:parameterization}), we use a compact parameterization that enforces the
structural constraints required by
Sections~\ref{sec:method:afm}--\ref{sec:method:objective}. The texture schedule
should satisfy the endpoint constraints
\(t_{\tex}(0;\rho)=0\) and \(t_{\tex}(1;\rho)=1\), be monotone so that local
texture time never runs backward, be semantic-leading so that
\(t_{\tex}(\tau;\rho)\le \tau\), and provide an easily computed derivative
\(t_{\tex}'(\tau;\rho)\). The derivative is needed for the inference ODE, the
Jacobian-corrected flow loss, and the kinetic regularizer.

\subsection{Semantic-leading property}
\label{app:semantic-leading}

Differentiating Eq.~\eqref{eq:fprime} gives
\(
  t_{\tex}''(\tau;\rho)
  =
  Z_\rho^{-1}
  \sum_{m=1}^{M} m a_m \tau^{m-1}
  \ge 0,
  \
  \tau\in[0,1],
\)
so \(t_{\tex}(\cdot;\rho)\) is convex. Since
\(t_{\tex}(0;\rho)=0\) and \(t_{\tex}(1;\rho)=1\), convexity implies that the
schedule lies below the chord connecting its endpoints:
\(
  t_{\tex}(\tau;\rho)\le \tau,
  \
  \forall \tau\in[0,1].
\)

\subsection{Bilevel formulation of schedule learning}
\label{app:bilevel}

Schedule learning is naturally a bilevel problem:
\begin{equation}
  \label{eq:bilevel-schedule-learning}
  \rho^\star
  \in
  \arg\min_\rho
  \mathcal{L}_{\rho}
  \big(
    \theta^\star(\rho),
    \rho
  \big),
  \qquad
  \theta^\star(\rho)
  \in
  \arg\min_\theta
  \mathcal{L}_{\theta}(\theta,\rho).
\end{equation}
The inner problem adapts the denoiser to a fixed schedule. The outer problem
then chooses the schedule using the adapted denoiser.

\subsection{Two-stage schedule probe}
\label{app:algorithm}

Algorithm~\ref{alg:probe} summarizes the two-stage procedure of Section~\ref{sec:method:probe}.

\begin{algorithm}[!htbp]
\caption{Two-stage schedule probe and fixed-schedule training}
\label{alg:probe}
\begin{algorithmic}[1]
\Require Main training budget \(S_{\mathrm{train}}\), probe fraction
\(\eta_{\mathrm{probe}}\approx 0.01\), burn-in steps \(S_{\mathrm{burn}}\)

\Statex \textbf{Stage I: joint schedule probe}
\State \textbf{Train:} temporary denoiser \(\theta_{\mathrm{probe}}\), schedule
parameters \(\rho\)
\State \textbf{Objective:}
\[
  \mathcal{L}_{\mathrm{probe}}
  \big(
    \theta_{\mathrm{probe}},
    \rho
  \big)
  =
  \mathcal{L}_{\rho}
  \big(
    \theta_{\mathrm{probe}},
    \rho
  \big)
\]
\State Set
\[
  S_{\mathrm{probe}}
  \leftarrow
  \left\lceil
    \eta_{\mathrm{probe}} S_{\mathrm{train}}
  \right\rceil
\]
\State Initialize \(\theta_{\mathrm{probe}}\)
\State Initialize \(\rho\) with \(t_{\tex}(\tau;\rho)\approx \tau\)
\State Initialize schedule buffer \(\mathcal{B}\leftarrow\emptyset\)

\For{\(s=1,\ldots,S_{\mathrm{probe}}\)}
  \State Sample \((x_1,c)\), Gaussian noise \(x_0\), and
  \(\tau\sim\mathcal{U}(0,1)\)
  \State Update both \(\theta_{\mathrm{probe}}\) and \(\rho\) using
  \[
    \nabla_{\theta_{\mathrm{probe}},\rho}
    \mathcal{L}_{\mathrm{probe}}
    \big(
      \theta_{\mathrm{probe}},
      \rho
    \big)
  \]
  \If{\(s>S_{\mathrm{burn}}\)}
    \State Append \(\rho\) to \(\mathcal{B}\)
  \EndIf
\EndFor

\State Average stable schedule parameters:
\[
  \bar\rho
  \leftarrow
  |\mathcal{B}|^{-1}
  \sum_{\rho_s\in\mathcal{B}}
  \rho_s
\]
\State Define and freeze the learned texture schedule
\[
  t_{\tex}^{\star}(\tau)
  \leftarrow
  t_{\tex}(\tau;\bar\rho)
\]
\State Discard \(\theta_{\mathrm{probe}}\)

\Statex \textbf{Stage II: fixed-schedule final training}
\State \textbf{Train:} final denoiser \(\theta\)
\State \textbf{Freeze:} schedule parameters \(\bar\rho\)
\State \textbf{Objective:}
\[
  \mathcal{L}_{\theta}(\theta,\bar\rho)
\]
\State Initialize final denoiser \(\theta\) from scratch

\For{\(s=1,\ldots,S_{\mathrm{train}}\)}
  \State Sample \((x_1,c)\), Gaussian noise \(x_0\), and
  \(\tau\sim\mathcal{U}(0,1)\)
  \State Update only \(\theta\) using
  \[
    \nabla_\theta
    \mathcal{L}_{\theta}(\theta,\bar\rho)
  \]
\EndFor
\end{algorithmic}
\end{algorithm}

\section{Additional Experiments and Analyses}
\label{app:additional}

Unless stated otherwise, the experiments in this section train the final model
from scratch at the 70K and 120K budgets of Table~\ref{tab:fid-noguide} and
report unguided FID-50K. These budgets count only the updates of the final
denoiser; the separate 10K-step probe is not included, and its denoiser is
discarded rather than used to warm-start the final model.

\begin{table}[!htbp]
\centering
\caption{\textbf{Ablations.} Unguided FID-50K of the final model at 70K and
120K iterations.}
\label{tab:ablation}
\small
\setlength{\tabcolsep}{8pt}
\begin{tabular}{lcc}
\toprule
Schedule and training objective & 70K & 120K \\
\midrule
\multicolumn{3}{@{}l}{\textit{Schedule learning and the corrected objective}} \\
SFD fixed offset                                             & 8.79 & 6.22 \\
SFD fixed offset + uniform local-time weighting              & 8.20 & 5.51 \\
Learned \(f^\star\) without Eq.~\eqref{eq:corrected-weights}  & 8.87 & 6.01 \\
Learned \(f^\star\) (Ours)                                   & \textbf{6.89} & 4.93 \\
\midrule
\multicolumn{3}{@{}l}{\textit{Learned schedule family}} \\
Non-convex (Bernstein, degree 9)                             & 9.74 & 6.50 \\
Convex, \(M=4\) (Ours)                                       & \textbf{6.89} & 4.93 \\
Convex, \(M=5\)                                              & 7.33 & \textbf{4.80} \\
Convex, \(M=6\)                                              & 7.66 & 5.13 \\
\bottomrule
\end{tabular}
\end{table}

\subsection{Schedule learning and the corrected objective}
\label{app:ablation-details}

Ours differs from SFD both in the learned schedule and in the corrected texture
weight of Eq.~\eqref{eq:corrected-weights}. To separate the two, we apply the
equivalent correction to SFD's fixed offset and train the learned \(f^\star\)
without it. Eq.~\eqref{eq:corrected-weights} cannot be copied literally to SFD,
because SFD's path contains clamped intervals at both ends. We therefore write
the equivalent correction in terms of SFD's global clock: each branch is
weighted only while its own local time is moving, so that its active local-time
range remains uniformly weighted. With normalized global clock \(u\in[0,1]\)
and offset \(\delta=0.3\),
\begin{equation}
\begin{aligned}
  t_{\sem}(u)&=\min((1+\delta)u,1),
  &
  t_{\tex}(u)&=\max((1+\delta)u-\delta,0),
  \\
  J_{\sem}(u)&=(1+\delta)\mathbf{1}[u<\tfrac{1}{1+\delta}],
  &
  J_{\tex}(u)&=(1+\delta)\mathbf{1}[u>\tfrac{\delta}{1+\delta}],
\end{aligned}
\end{equation}
and \(J_{\sem}\), \(J_{\tex}\) multiply the semantic and texture losses. While
a branch is active, its loss is multiplied by \(1+\delta=1.3\); while it is
clamped, its duplicated endpoint loss receives zero weight. The correction is
applied to both branches because SFD freezes texture at the beginning and
semantics at the end.

The correction improves fixed SFD from 6.22 to 5.51 at 120K, a general benefit
of the reweighting. The learned \(f^\star\) without the correction also
improves over SFD (6.01 vs.\ 6.22), but this comparison changes the
semantic--texture pairing and the implicitly induced texture-time training
density at once. Holding the density treatment fixed, schedule learning
improves FID from 5.51 to 4.93. The correction is especially well matched to
the curved learned schedule: because \(f^{\star\prime}(\tau)\) varies over
time, uniform-\(\tau\) sampling would otherwise change the texture-time density
and under-train the late texture region, giving 8.87 instead of 6.89 at 70K.
Schedule learning and the corrected objective are therefore complementary:
learning changes which semantic state is paired with each texture state, and
the correction ensures that this change in ordering does not also change
texture-time coverage.

\subsection{Schedule family and polynomial degree}

The non-convex comparison replaces the convex degree-4 map with a degree-9
Bernstein parameterization,
\begin{equation}
\begin{aligned}
  f'(\tau)&=\frac{1}{Z}\sum_{m=0}^{9}a_mB_{m,9}(\tau),
  &
  B_{m,9}(\tau)&=\binom{9}{m}\tau^m(1-\tau)^{9-m},
  \\
  a_m&=\operatorname{softplus}(r_m)>0,
  &
  Z&=\frac{1}{10}\sum_{m=0}^{9}a_m,
\end{aligned}
\end{equation}
with \(f(\tau)=\int_0^\tau f'(s)\,ds\) in closed form. This guarantees
\(f(0)=0\), \(f(1)=1\), and \(f'>0\), while the free Bernstein profile allows
\(f'\) to increase or decrease, so \(f\) need not be convex. The
semantic-leading condition \(f(\tau)\le\tau\) is enforced separately during the
probe. The probe, kinetic regularization, schedule learning rate, stable-window
averaging, final training, and evaluation protocol are unchanged. The learned
schedule is genuinely non-convex: it stays nearly synchronous early, delays
texture mainly in the middle, and catches up near the endpoint
(Figure~\ref{fig:ablation-curves}). It reaches 9.74 and 6.50, against 6.89 and
4.93 for the convex schedule, which supports convexity as a useful
SFD-informed inductive bias in our setting.

Repeating the complete procedure with \(M=5\) and \(M=6\), each with a new
probe and a final model trained from scratch, gives 7.33/4.80 and 7.66/5.13 at
70K/120K. The early 70K results vary with \(M\), while at 120K all three
degrees lie between 4.80 and 5.13, and every degree improves over SFD at both
budgets. We use \(M=4\), the smallest degree that produced a stable,
non-collapsed schedule.

\begin{figure}[!htbp]
  \centering
  \includegraphics[width=\linewidth]{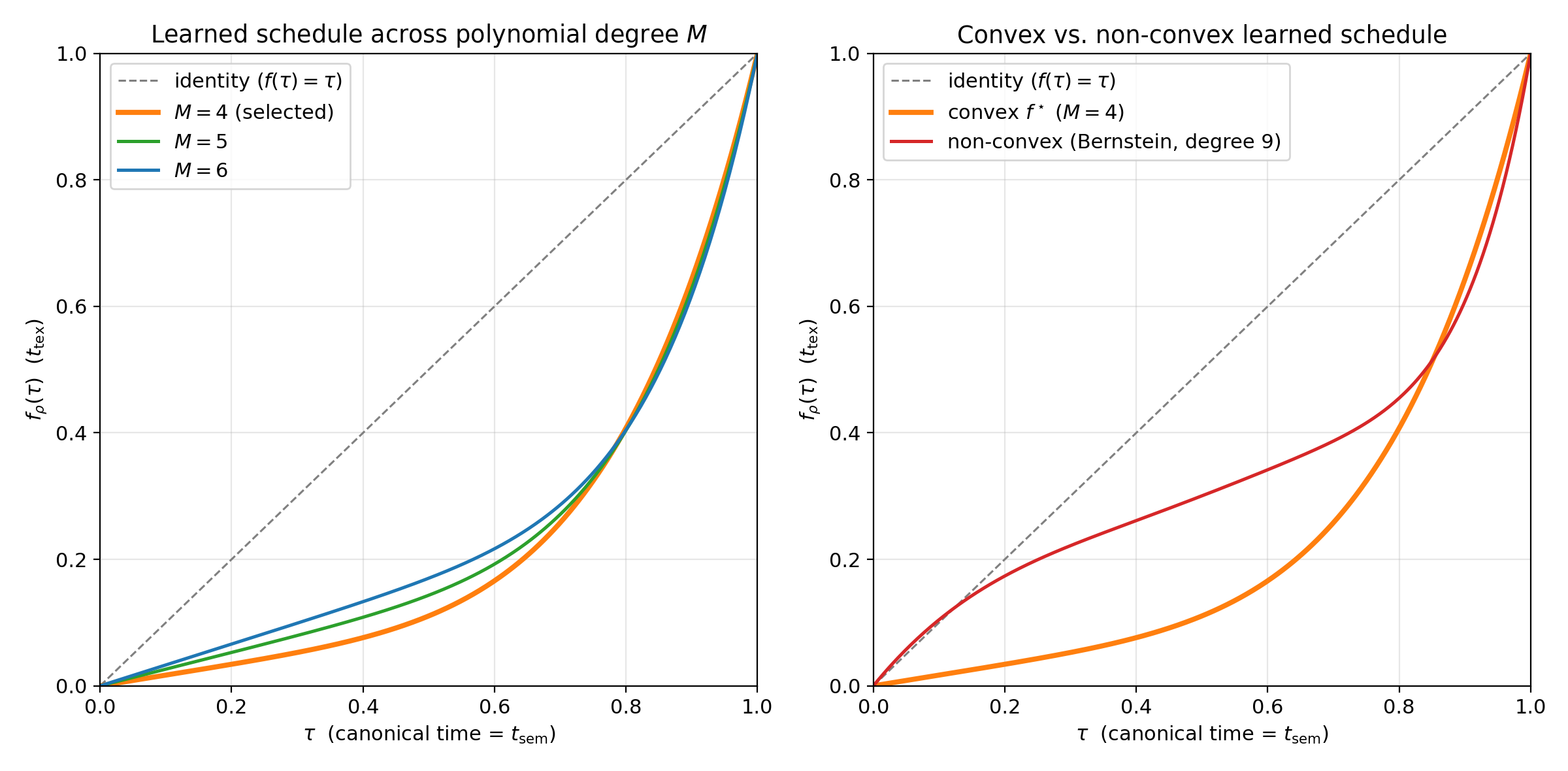}
  \caption{\textbf{Learned schedules in the ablations.} Left: probe-learned
  convex schedules for polynomial degrees \(M\in\{4,5,6\}\). Right: the convex
  \(f^\star\) and the probe-learned non-convex Bernstein schedule.}
  \label{fig:ablation-curves}
\end{figure}

\subsection{Interpretation of the learned schedule}

Writing the semantic lead as \(g(\tau)=\tau-f(\tau)\), convexity of \(f\) with
fixed endpoints makes \(g\) nonnegative and concave, with \(g(0)=g(1)=0\): the
lead is built and then removed once. This is the coarse geometry of SFD's
offset, whose semantic lead is introduced, maintained, and removed once rather
than repeatedly expanding and contracting; convexity is a compact proxy for
these semantic-first dynamics. What the probe learns is therefore not whether
semantics should lead, but how the lead is distributed over time. The final
SemVAE schedule builds a lead of about \(0.44\) around two-thirds of the
trajectory (\(\tau\approx0.67\)) and then closes it smoothly, rather than
holding SFD's fixed \(0.3\) plateau between piecewise-linear transitions.

The same smoothness offers a reading of the early-training advantage in
Table~\ref{tab:fid-noguide}. At the matched 70K and 120K budgets, Ours and SFD
train the same architecture for the same number of updates under two frozen
schedules. SFD's clamped offset is piecewise linear, with kinks at the clamp
boundaries and intervals in which one representation's local time is frozen,
whereas \(f^\star\) is continuously differentiable and yields a smoothly
varying semantic--texture coupling; the early advantage is consistent with a
conditional vector field that is easier to fit.

\subsection{Spatial statistics of the semantic representations}

An aggressive schedule means a smaller \(f(\tau)\): semantics is farther
denoised while texture remains at a noisier local time, so the probe learns how
much semantic head start is needed before semantic conditioning helps texture
prediction. The relevant distinction is local, position-aligned information
useful for texture prediction, not global semantic strength. SemVAE retains
some such detail, so a partially resolved semantic state already provides local
cues and texture need not wait long. The PCA representations are smoother and
more global; their local values carry little direct texture information, so
texture benefits mainly once the semantic field is cleaner and more coherent,
inducing a larger delay.

Table~\ref{tab:spatial-stats} quantifies this with two statistics of the
standardized \(16\times16\) semantic maps: the lag-1 spatial autocorrelation,
which is high for smooth maps, and the low-pass \(R^2\), the fraction of
variance that survives \(16\to8\to16\) down/up-sampling. Both increase
monotonically with schedule aggressiveness: representations richer in local,
texture-relevant information require less semantic lead.

\begin{table}[!htbp]
\centering
\caption{\textbf{Spatial statistics of the semantic representations} and the
shape of the learned schedule (Figure~\ref{fig:schedule-repr}).}
\label{tab:spatial-stats}
\small
\setlength{\tabcolsep}{8pt}
\begin{tabular}{lccl}
\toprule
Semantic latent & Lag-1 autocorr. & Low-pass \(R^2\) & Learned schedule \\
\midrule
SemVAE   & 0.59 & 0.67 & closest to identity \\
DINO-PCA & 0.75 & 0.79 & intermediate \\
CLIP-PCA & 0.80 & 0.83 & most aggressive \\
\bottomrule
\end{tabular}
\end{table}

\subsection{Reading of the schedule objective}
\label{app:reading}

The schedule objective \(\mathcal{L}_{\rho}\) in Eq.~\eqref{eq:rho-objective}
can be read as the Lagrangian form of a constrained problem: find a coupling
with low velocity-prediction error while keeping the global-time trajectory
energy within a stable range. Under squared loss, the population-optimal
predictor is the conditional mean, and its minimum risk is the total
conditional variance,
\begin{equation}
  \inf_h\mathbb{E}\big[\|v-h(z_\tau,\tau,c)\|^2\big]
  =
  \mathbb{E}\big[\operatorname{tr}\operatorname{Cov}(v\mid z_\tau,\tau,c)\big],
\end{equation}
where \(v\) is the group-velocity target, \(z_\tau\) is the joint noisy
semantic--texture state, and the infimum runs over all predictors \(h\).
Because the corrected weights hold the local-time density fixed, a lower minimum of
\(\mathcal{L}_{\mathrm{flow}}\) identifies a coupling under which the required
velocities are less ambiguous from the model's input, rather than a schedule
that merely reweights noise levels.

Predictability alone does not account for the numerical cost of traversing the
path with finite NFEs. An extreme semantic-first schedule may make texture
prediction easier while postponing nearly all texture motion to a narrow
interval, where \(t_{\tex}'(\tau;\rho)\) and the global-time speed become large.
\(\mathcal{R}_{\mathrm{kin}}\) is the squared-speed energy of the predicted
global-time velocity along the training interpolation and directly penalizes such concentration. This matters because both
branches share one global clock and one NFE budget, so the solver cannot refine
the semantic and texture time grids independently. We chose the kinetic energy
because it targets the observed collapse; schedule-only alternatives such as an
\(L_2\) penalty on the deviation of \(t_{\tex}'\) from one or total variation do
not account for the denoiser-dependent latent velocity, and a penalty on the
ODE condition number would be substantially more expensive to estimate.

\subsection{Sampling variability of the headline results}
\label{app:variance}

We generate three independent 50K class-balanced sample sets from each fixed
checkpoint, varying only the sampling seed. For every run, we generate 50
samples per ImageNet class using sampling seeds \((0,1,2)\), the same ADM
reference statistics, a byte-identical FID implementation, the same solver, and
the same 2-GPU layout, with each method's released sampling implementation and
checkpoints and the corresponding weak checkpoint for AutoGuidance. Ours is
evaluated with AutoGuidance scale \(w=1.4\) in this table and SFD-XL with its
released sampling configuration. The standard deviations in
Table~\ref{tab:variance} are sampling-seed variation for fixed checkpoints.

\begin{table}[!htbp]
\centering
\caption{\textbf{FID-50K under three sampling seeds} for the 1M-iteration
Ours-XL checkpoint and the released 4M-iteration SFD-XL checkpoint. Ours uses
AutoGuidance scale \(w=1.4\).}
\label{tab:variance}
\small
\setlength{\tabcolsep}{7pt}
\begin{tabular}{lcccc}
\toprule
& seed 0 & seed 1 & seed 2 & mean \(\pm\) std \\
\midrule
Ours-XL 1M, unguided              & 2.364 & 2.407 & 2.439 & \(2.404 \pm 0.038\) \\
SFD-XL 4M, unguided               & 2.607 & 2.630 & 2.606 & \(2.614 \pm 0.014\) \\
\midrule
Ours-XL 1M, AutoGuidance (\(w=1.4\)) & 1.035 & 1.013 & 1.020 & \(1.023 \pm 0.011\) \\
SFD-XL 4M, AutoGuidance           & 1.065 & 1.046 & 1.059 & \(1.056 \pm 0.010\) \\
\bottomrule
\end{tabular}
\end{table}

All three evaluations of Ours (2.364--2.439) are lower than all three
evaluations of the released SFD-XL checkpoint (2.606--2.630), with a minimum
separation of 0.167. With AutoGuidance, the highest FID among our three runs
(1.035) remains below the lowest among the three SFD-XL runs (1.046). Our seed-0
unguided value is within 0.01 of the 2.37 in Table~\ref{tab:fid-noguide}. The
released SFD-XL checkpoint gives \(2.614 \pm 0.014\) in this re-evaluation,
compared with the 2.54 reported by \citet{sfd} and listed in
Table~\ref{tab:fid-noguide}.

\section{Proofs for Main Results}
\label{s:proofs}

\subsection{Proof of Theorem \ref{thm:ideal-async-flow}}
\label{ss:proof:thm:ideal-async-flow}
In addition to claims (I) and (II) of Theorem~\ref{thm:ideal-async-flow}, we state and prove the explicit form of the scheduled distribution and the score form of the global flow:

\smallskip
\noindent\textbf{Interpolating distribution.}
The scheduled distribution \(p_\tau(\cdot\mid c)\) has the explicit form
\begin{equation}
  \label{eq:async-density-convolution}
  p_\tau(\cdot\mid c)
  =
  \big(
    (A_\tau)_{\#}p_1(\cdot\mid c)
  \big)
  *
  \mathcal{N}(0,\Sigma_\tau),
\end{equation}
where $
  A_\tau
  =
  \mathrm{diag}
  \big(
    t_{\tex}(\tau)I_{d_{\tex}},
    t_{\sem}(\tau)I_{d_{\sem}}
  \big),
  $ and $
  \Sigma_\tau
  =
  \mathrm{diag}
  \big(
    (1-t_{\tex}(\tau))^2I_{d_{\tex}},
    (1-t_{\sem}(\tau))^2I_{d_{\sem}}
  \big),
$
and \((A_\tau)_{\#}p_1\) denotes the pushforward of \(p_1\) under
\(x\mapsto A_\tau x\). Thus \(p_\tau\) starts at the base Gaussian distribution
when \(\tau=0\) and ends at the clean data-latent distribution when
\(\tau=1\), with endpoint identities understood as weak limits.

Consequently, the ideal global flow can
be written in terms of the overall score of \(p_\tau\) as
\begin{equation}
  \label{eq:score-transform-global}
  \begin{aligned}
  V^\star(z,\tau,c)
  =
  \bigg[
  &
    \frac{t_{\tex}'(\tau)}{t_{\tex}(\tau)}z^{\tex}
    +
    \frac{t_{\tex}'(\tau)(1-t_{\tex}(\tau))}{t_{\tex}(\tau)}
    \nabla_{z^{\tex}}\log p_\tau(z\mid c),
  \\
  &
    \frac{t_{\sem}'(\tau)}{t_{\sem}(\tau)}z^{\sem}
    +
    \frac{t_{\sem}'(\tau)(1-t_{\sem}(\tau))}{t_{\sem}(\tau)}
    \nabla_{z^{\sem}}\log p_\tau(z\mid c)
  \bigg].
  \end{aligned}
\end{equation}

\paragraph{Proof sketch.}
The interpolating distribution follows by writing the asynchronous state as an
anisotropically scaled data latent plus anisotropic Gaussian noise. The
continuity equation follows by differentiating expectations of smooth test
functions along the asynchronous path and conditioning the sample-wise velocity
on the observed state. The score identities follow from Tweedie's identity
applied separately to the texture and semantic Gaussian channels.

We prove the three claims in order. Throughout the proof, all distributions are
conditioned on the class label \(c\), and we suppress this conditioning when it
is clear from context.

\paragraph{Interpolating distribution.}
For a fixed global time \(\tau\), define
\[
  A_\tau
  =
  \mathrm{diag}
  \big(
    t_{\tex}(\tau)I_{d_{\tex}},
    t_{\sem}(\tau)I_{d_{\sem}}
  \big),
  \qquad
  B_\tau
  =
  \mathrm{diag}
  \big(
    (1-t_{\tex}(\tau))I_{d_{\tex}},
    (1-t_{\sem}(\tau))I_{d_{\sem}}
  \big).
\]
By the component-wise interpolation in Eq.~\eqref{eq:component-interpolant},
the asynchronous state can be written as
\begin{equation}
  \label{eq:appendix-async-affine-form}
  z(\tau)
  =
  A_\tau x_1 + B_\tau x_0,
\end{equation}
where \(x_1\sim p_1(\cdot\mid c)\) and
\(x_0\sim\mathcal{N}(0,I_{d_{\tex}+d_{\sem}})\). Therefore, conditional on
\(x_1\), the random variable \(z(\tau)\) is Gaussian:
\[
  z(\tau)\mid x_1
  \sim
  \mathcal{N}
  \left(
    A_\tau x_1,
    B_\tau B_\tau^\top
  \right).
\]
Since
\[
  B_\tau B_\tau^\top
  =
  \mathrm{diag}
  \big(
    (1-t_{\tex}(\tau))^2I_{d_{\tex}},
    (1-t_{\sem}(\tau))^2I_{d_{\sem}}
  \big)
  =
  \Sigma_\tau,
\]
marginalizing over \(x_1\sim p_1(\cdot\mid c)\) gives
\[
  p_\tau(\cdot\mid c)
  =
  \big(
    (A_\tau)_{\#}p_1(\cdot\mid c)
  \big)
  *
  \mathcal{N}(0,\Sigma_\tau),
\]
which is Eq.~\eqref{eq:async-density-convolution}. At \(\tau=0\),
\(A_\tau=0\) and \(\Sigma_\tau=I\), so \(p_\tau=p_0\). At \(\tau=1\),
\(A_\tau=I\) and \(\Sigma_\tau=0\), so \(p_\tau=p_1\), with the latter identity
understood as a weak limit.

\paragraph{Continuity equation.}
For each endpoint pair \((x_0,x_1)\), differentiating
Eq.~\eqref{eq:scheduled-async-state-general} with respect to global time gives
\begin{equation}
  \label{eq:appendix-sample-global-velocity}
  \dot z(\tau)
  =
  \left[
    t_{\tex}'(\tau)u^{\tex}(t_{\tex}(\tau)),
    \;
    t_{\sem}'(\tau)u^{\sem}(t_{\sem}(\tau))
  \right].
\end{equation}
Let \(\varphi\in C_c^\infty(\mathbb{R}^{d_{\tex}+d_{\sem}})\) be a smooth test
function. Then
\begin{align}
  \frac{d}{d\tau}
  \mathbb{E}
  \left[
    \varphi(z(\tau))
    \mid c
  \right]
  &=
  \mathbb{E}
  \left[
    \nabla_z\varphi(z(\tau))^\top
    \dot z(\tau)
    \mid c
  \right].
\end{align}
Taking conditional expectation of the sample-wise global-time velocity given
\(z(\tau)=z\) yields
\begin{align}
  \mathbb{E}
  \left[
    \dot z(\tau)
    \mid
    z(\tau)=z,\; c
  \right]
  &=
  \left[
    t_{\tex}'(\tau)
    \mathbb{E}
    \left[
      u^{\tex}(t_{\tex}(\tau))
      \mid
      z(\tau)=z,\; c
    \right],
    \right.
  \\
  &\qquad\qquad
    \left.
    t_{\sem}'(\tau)
    \mathbb{E}
    \left[
      u^{\sem}(t_{\sem}(\tau))
      \mid
      z(\tau)=z,\; c
    \right]
  \right]
  \\
  &=
  V^\star(z,\tau,c),
\end{align}
where the last equality uses Eq.~\eqref{eq:ideal-local-flow}. Therefore,
\[
  \frac{d}{d\tau}
  \int
    \varphi(z)p_\tau(z\mid c)
  \,dz
  =
  \int
    \nabla_z\varphi(z)^\top
    V^\star(z,\tau,c)
    p_\tau(z\mid c)
  \,dz .
\]
Integrating by parts gives
\[
  \frac{d}{d\tau}
  \int
    \varphi(z)p_\tau(z\mid c)
  \,dz
  =
  -
  \int
    \varphi(z)
    \nabla_z\!\cdot
    \left(
      p_\tau(z\mid c)V^\star(z,\tau,c)
    \right)
  \,dz .
\]
Since this holds for all smooth compactly supported test functions
\(\varphi\), \(p_\tau\) satisfies the continuity equation in the weak sense:
\[
  \partial_\tau p_\tau(z\mid c)
  +
  \nabla_z\!\cdot
  \left(
    p_\tau(z\mid c)V^\star(z,\tau,c)
  \right)
  =
  0.
\]
Under the usual regularity assumptions ensuring existence and uniqueness of the
ODE flow, the probability-flow ODE
\[
  \frac{dZ_\tau}{d\tau}
  =
  V^\star(Z_\tau,\tau,c)
\]
has marginals \(p_\tau(\cdot\mid c)\). Since \(p_0\) is Gaussian and \(p_1\) is
the clean data-latent distribution, this ODE transports noise to data.

\paragraph{Score characterization.}
Fix local times \((t_{\tex},t_{\sem})\) and a group
\(g\in\{\tex,\sem\}\). Write \(t_g=t_{\tex}\) if \(g=\tex\) and
\(t_g=t_{\sem}\) if \(g=\sem\). For the linear Gaussian noising channel,
\[
  z^g
  =
  t_g x_1^g + (1-t_g)x_0^g,
  \qquad
  x_0^g\sim\mathcal{N}(0,I).
\]
Let
\[
  s_g(z,t_{\tex},t_{\sem},c)
  =
  \nabla_{z^g}
  \log p_{t_{\tex},t_{\sem}}(z\mid c).
\]
Differentiating the conditional Gaussian density with respect to \(z^g\) and
then averaging over the posterior of \(x_1\) gives the group-wise Tweedie
identity
\begin{equation}
  \label{eq:appendix-tweedie-noise}
  s_g(z,t_{\tex},t_{\sem},c)
  =
  -
  \frac{
    \mathbb{E}
    \left[
      x_0^g
      \mid
      z(t_{\tex},t_{\sem})=z,\; c
    \right]
  }{
    1-t_g
  }.
\end{equation}
Equivalently,
\begin{equation}
  \label{eq:appendix-noise-posterior-mean}
  \mathbb{E}
  \left[
    x_0^g
    \mid
    z(t_{\tex},t_{\sem})=z,\; c
  \right]
  =
  -(1-t_g)s_g(z,t_{\tex},t_{\sem},c).
\end{equation}
Using \(z^g=t_gx_1^g+(1-t_g)x_0^g\), we also have
\begin{align}
  \mathbb{E}
  \left[
    x_1^g
    \mid
    z(t_{\tex},t_{\sem})=z,\; c
  \right]
  &=
  \frac{
    z^g
    -
    (1-t_g)
    \mathbb{E}
    \left[
      x_0^g
      \mid
      z(t_{\tex},t_{\sem})=z,\; c
    \right]
  }{
    t_g
  }
  \\
  &=
  \frac{
    z^g
    +
    (1-t_g)^2
    s_g(z,t_{\tex},t_{\sem},c)
  }{
    t_g
  }.
\end{align}
Therefore the ideal local flow is
\begin{align}
  v_g^\star(z,t_{\tex},t_{\sem},c)
  &=
  \mathbb{E}
  \left[
    x_1^g-x_0^g
    \mid
    z(t_{\tex},t_{\sem})=z,\; c
  \right]
  \\
  &=
  \frac{z^g}{t_g}
  +
  \frac{(1-t_g)^2}{t_g}
  s_g(z,t_{\tex},t_{\sem},c)
  +
  (1-t_g)
  s_g(z,t_{\tex},t_{\sem},c)
  \\
  &=
  \frac{z^g}{t_g}
  +
  \frac{1-t_g}{t_g}
  s_g(z,t_{\tex},t_{\sem},c).
\end{align}
Specializing this identity to \(g=\tex\) and to \(g=\sem\) gives
Eq.~\eqref{eq:score-transform-tex}. Finally, along the scheduled path,
\(p_\tau=p_{t_{\tex}(\tau),t_{\sem}(\tau)}\). Substituting the two local score
forms into Eq.~\eqref{eq:global-target-field} gives
Eq.~\eqref{eq:score-transform-global}.

\subsection{Proof of Theorem \ref{thm:async-fm-population-optimum}}
\label{ss:proof:thm:async-fm-population-optimum}
\begin{theorem}[Population optimum of asynchronous flow matching]
\label{thm:async-fm-population-optimum}
Fix a continuously differentiable, nondecreasing schedule \(t_{\tex}(\cdot;\rho)\).
Assume \(\omega_{\tex}(\tau,\rho)>0\) and \(\omega_{\sem}(\tau,\rho)>0\) depend
only on the sampled local times, not on the endpoint pair \((x_0,x_1)\), and
that the weighted risk
\(\mathbb{E}\big[\sum_g\omega_g(\tau,\rho)\|u^g\|_g^2\big]\) is finite. If each
\(\hat v^{g}\) ranges over all measurable functions of
\((z,t_{\tex},t_{\sem},c)\), then any minimizer over \(\theta\) of
Eq.~\eqref{eq:generic-async-fm-loss} satisfies
\begin{equation}
  \label{eq:population-regression-optimum-tex}
  \hat v_{\theta^\star}^{g}
  \big(
    z,
    t_{\tex}(\tau;\rho),
    \tau,
    c
  \big)
  =
  v_{g}^\star
  \big(
    z,
    t_{\tex}(\tau;\rho),
    \tau,
    c
  \big),
  \qquad
  g\in\{\tex,\sem\},
\end{equation}
for \(p_\tau^\rho(z\mid c)\)-almost every \(z\) and almost every \(\tau\).
\end{theorem}

\paragraph{Proof sketch.}
For each fixed pair of local times, the squared-error minimizer is the
conditional mean of the corresponding local-time velocity. Positive time weights
change how local-time pairs are averaged, but not this pointwise conditional
mean.

We prove the statement for the texture component; the semantic component is
identical. Fix a schedule \(t_{\tex}(\tau;\rho)\). For a fixed global time
\(\tau\), write
\[
  T_{\tex}=t_{\tex}(\tau;\rho),
  \qquad
  T_{\sem}=\tau,
  \qquad
  Z=z(T_{\tex},T_{\sem}).
\]
The texture component of the weighted flow objective at this time is
\[
  \omega_{\tex}(\tau,\rho)
  \mathbb{E}
  \left[
    \left\|
      \hat v_\theta^{\tex}(Z,T_{\tex},T_{\sem},c)
      -
      u^{\tex}(T_{\tex})
    \right\|_{\tex}^2
  \right].
\]
Since \(\omega_{\tex}(\tau,\rho)>0\) and depends only on the sampled local
times, it does not change the pointwise minimizer with respect to the predicted
function. Thus it suffices to minimize
\[
  \mathbb{E}
  \left[
    \left\|
      h(Z,T_{\tex},T_{\sem},c)
      -
      u^{\tex}(T_{\tex})
    \right\|_{\tex}^2
  \right]
\]
over measurable functions \(h\). By the standard \(L_2\)-projection identity,
the minimizer is the conditional mean
\[
  h^\star(z,T_{\tex},T_{\sem},c)
  =
  \mathbb{E}
  \left[
    u^{\tex}(T_{\tex})
    \mid
    Z=z,\; T_{\tex},\; T_{\sem},\; c
  \right].
\]
Because \(T_{\tex}\) and \(T_{\sem}\) are fixed by the sampled global time
\(\tau\), this conditional mean is exactly
\[
  v_{\tex}^\star(z,T_{\tex},T_{\sem},c)
  =
  v_{\tex}^\star
  \big(
    z,
    t_{\tex}(\tau;\rho),
    \tau,
    c
  \big),
\]
as defined in Eq.~\eqref{eq:ideal-local-flow}. Therefore, in the infinite-data
and infinite-capacity limit,
\[
  \hat v_{\theta^\star}^{\tex}
  \big(
    z,
    t_{\tex}(\tau;\rho),
    \tau,
    c
  \big)
  =
  v_{\tex}^\star
  \big(
    z,
    t_{\tex}(\tau;\rho),
    \tau,
    c
  \big)
\]
for \(p_\tau^\rho(z\mid c)\)-almost every \(z\) and almost every \(\tau\).

The same argument applied to the semantic component gives
\[
  \hat v_{\theta^\star}^{\sem}
  \big(
    z,
    t_{\tex}(\tau;\rho),
    \tau,
    c
  \big)
  =
  v_{\sem}^\star
  \big(
    z,
    t_{\tex}(\tau;\rho),
    \tau,
    c
  \big).
\]
Because the model is assumed to have infinite capacity, the two component
predictions can simultaneously realize their respective conditional means.
Positive time weights only determine how local-time pairs are averaged in the
global objective; they do not alter the pointwise conditional-mean target at any
fixed local-time pair.

\subsection{Proof of Lemma \ref{lem:schedule-time-reweighting}}
\label{ss:proof:lem:schedule-time-reweighting}

By definition,
\begin{equation}
  \label{eq:texture-global-time-average-proof}
  \mathbb{E}_{\tau\sim\mathcal{U}(0,1)}
  \left[
    \omega_{\tex}(\tau,\rho)
    \ell_{\tex}
    \big(
      \theta;
      t_{\tex}(\tau;\rho),
      \tau
    \big)
  \right]
  =
  \int_0^1
    \omega_{\tex}(\tau,\rho)
    \ell_{\tex}
    \big(
      \theta;
      t_{\tex}(\tau;\rho),
      \tau
    \big)
  \,d\tau .
\end{equation}
Apply the change of variables \(s=t_{\tex}(\tau;\rho)\). Since
\(t_{\tex}(\cdot;\rho)\) is a \(C^1\) bijection of \([0,1]\) with
\(t_{\tex}'>0\) on \((0,1)\),
\[
  \tau=t_{\tex}^{-1}(s;\rho),
  \qquad
  d\tau
  =
  \frac{
    ds
  }{
    t_{\tex}'
    \big(
      t_{\tex}^{-1}(s;\rho);
      \rho
    \big)
  } .
\]
Therefore the texture term becomes
\begin{equation}
  \label{eq:texture-local-time-decomposition}
  \int_0^1
  \frac{
    \omega_{\tex}
    \big(
      t_{\tex}^{-1}(s;\rho),
      \rho
    \big)
  }{
    t_{\tex}'
    \big(
      t_{\tex}^{-1}(s;\rho);
      \rho
    \big)
  }
  \ell_{\tex}
  \big(
    \theta;
    s,
    t_{\tex}^{-1}(s;\rho)
  \big)
  \,ds .
\end{equation}
For this to equal the first term of
Eq.~\eqref{eq:local-time-invariance-condition} for arbitrary
\(\ell_{\tex}\), the ratio multiplying \(\ell_{\tex}\) must equal one almost
everywhere. This gives
\(\omega_{\tex}(\tau,\rho)=t_{\tex}'(\tau;\rho)\). The semantic term requires no
change of variables because \(t_{\sem}(\tau)=\tau\):
\begin{equation}
  \label{eq:semantic-local-time-decomposition}
  \mathbb{E}_{\tau\sim\mathcal{U}(0,1)}
  \left[
    \omega_{\sem}(\tau,\rho)
    \ell_{\sem}
    \big(
      \theta;
      t_{\tex}(\tau;\rho),
      \tau
    \big)
  \right]
  =
  \int_0^1
    \omega_{\sem}(s,\rho)
    \ell_{\sem}
    \big(
      \theta;
      t_{\tex}(s;\rho),
      s
    \big)
  \,ds .
\end{equation}
Matching the second term of
Eq.~\eqref{eq:local-time-invariance-condition} for arbitrary \(\ell_{\sem}\)
therefore requires \(\omega_{\sem}(\tau,\rho)=1\) almost everywhere.

\subsection{Explicit forms of the corrected objective}
\label{ss:corrected-forms}

Plugging Eq.~\eqref{eq:corrected-weights} into Eq.~\eqref{eq:generic-async-fm-loss}, the corrected flow objective of Eq.~\eqref{eq:flow-objective} reads, equivalently,
\begin{equation}
  \label{eq:corrected-flow-objective-expanded}
  \mathcal{L}_{\mathrm{flow}}(\theta,\rho)
  =
  \mathbb{E}_{\tau\sim\mathcal{U}(0,1)}
  \left[
    w_{\tex}\,
    \operatorname{sg}
    \left(
      t_{\tex}'(\tau;\rho)
    \right)
    \ell_{\tex}
    \big(
      \theta;
      t_{\tex}(\tau;\rho),
      \tau
    \big)
    +
    w_{\sem}\,
    \ell_{\sem}
    \big(
      \theta;
      t_{\tex}(\tau;\rho),
      \tau
    \big)
  \right].
\end{equation}
Ignoring the stop-gradient annotation, Eq.~\eqref{eq:corrected-flow-objective-expanded}
has the local-time form
\begin{equation}
  \label{eq:corrected-local-time-objective}
  \mathcal{L}_{\mathrm{flow}}(\theta,\rho)
  =
  w_{\tex}
  \mathbb{E}_{s\sim\mathcal{U}(0,1)}
  \left[
    \ell_{\tex}
    \big(
      \theta;
      s,
      t_{\tex}^{-1}(s;\rho)
    \big)
  \right]
  +
  w_{\sem}
  \mathbb{E}_{s\sim\mathcal{U}(0,1)}
  \left[
    \ell_{\sem}
    \big(
      \theta;
      t_{\tex}(s;\rho),
      s
    \big)
  \right].
\end{equation}
The two forms have the same value for every \(\rho\). With the stop-gradient,
the schedule gradient is that of the global-time form in
Eq.~\eqref{eq:corrected-flow-objective-expanded} with \(t_{\tex}'\) held fixed.

\section{Additional Qualitative Results}
\label{app:qualitative}

We show additional class-conditional samples from the final Ours-XL
model at \(256\times256\), one figure per ImageNet class. All samples
are generated with AutoGuidance (\(w=1.5\)), matching the setting of
Table~\ref{tab:fid-ag}.

\begin{figure}[!htbp]
  \centering
  \includegraphics[width=\linewidth]{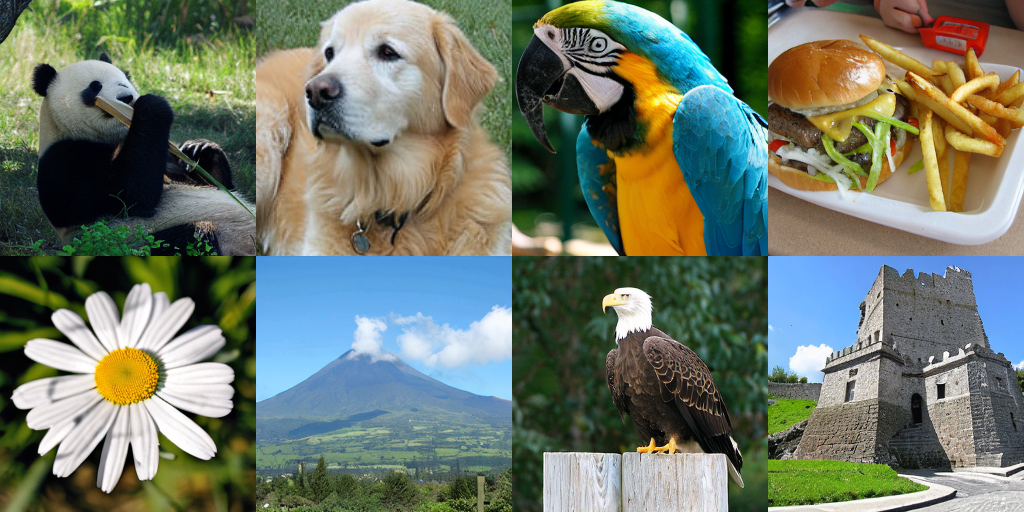}
  \caption{
    \textbf{Qualitative samples from Ours-XL trained at \(256\times256\) resolution.}
    We show selected class-conditional samples generated by the final Ours-XL
    model using the same AutoGuidance setting as in Table~\ref{tab:fid-ag}.
  }
  \label{fig:qualitative}
\end{figure}

\begin{figure}[!htbp]
  \centering
  \includegraphics[width=\linewidth]{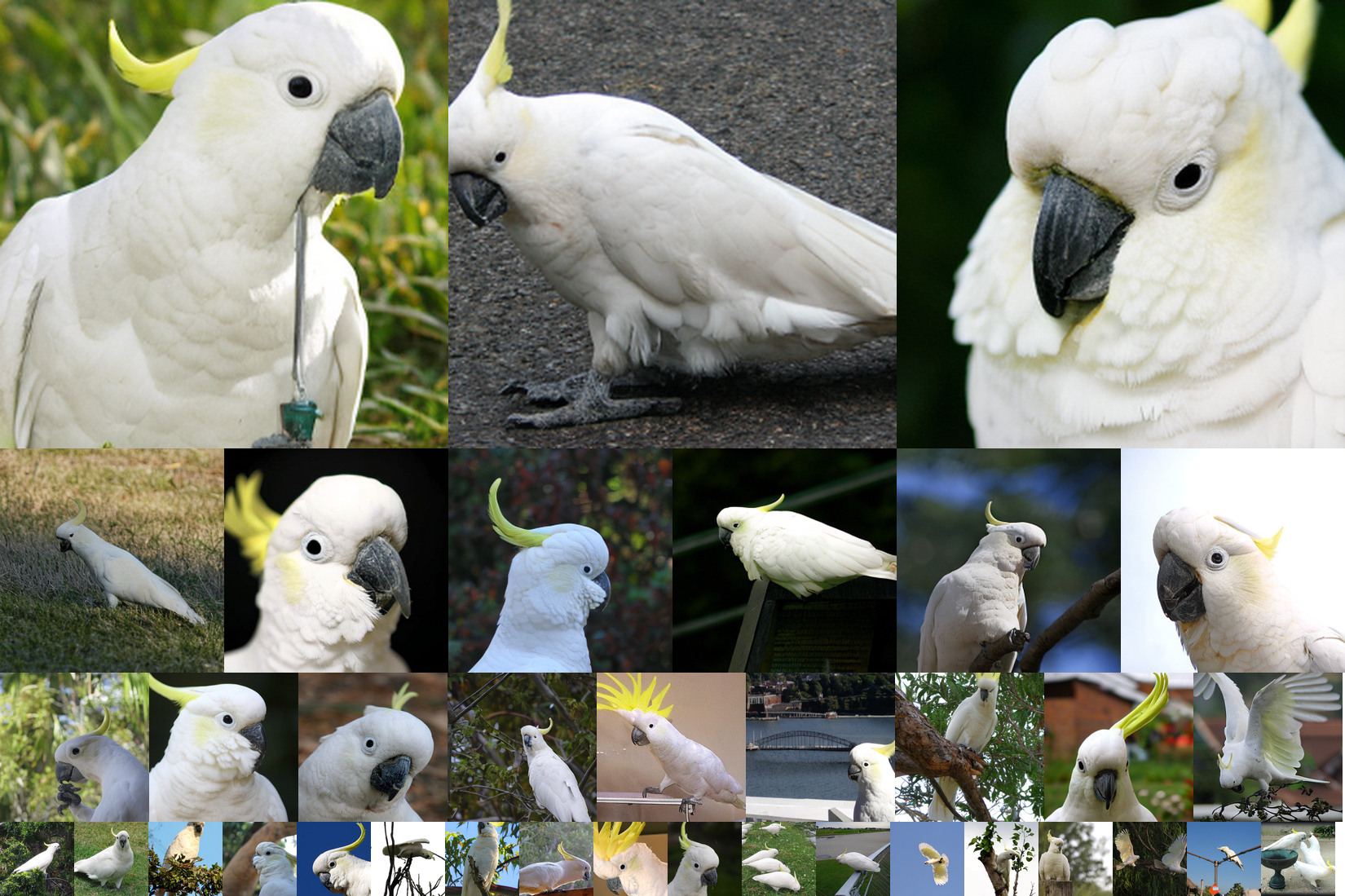}
  \caption{Cockatoo (class~89). Samples from Ours-XL with AutoGuidance, \(w=1.5\).}
  \label{fig:qual-cockatoo}
\end{figure}

\begin{figure}[!htbp]
  \centering
  \includegraphics[width=\linewidth]{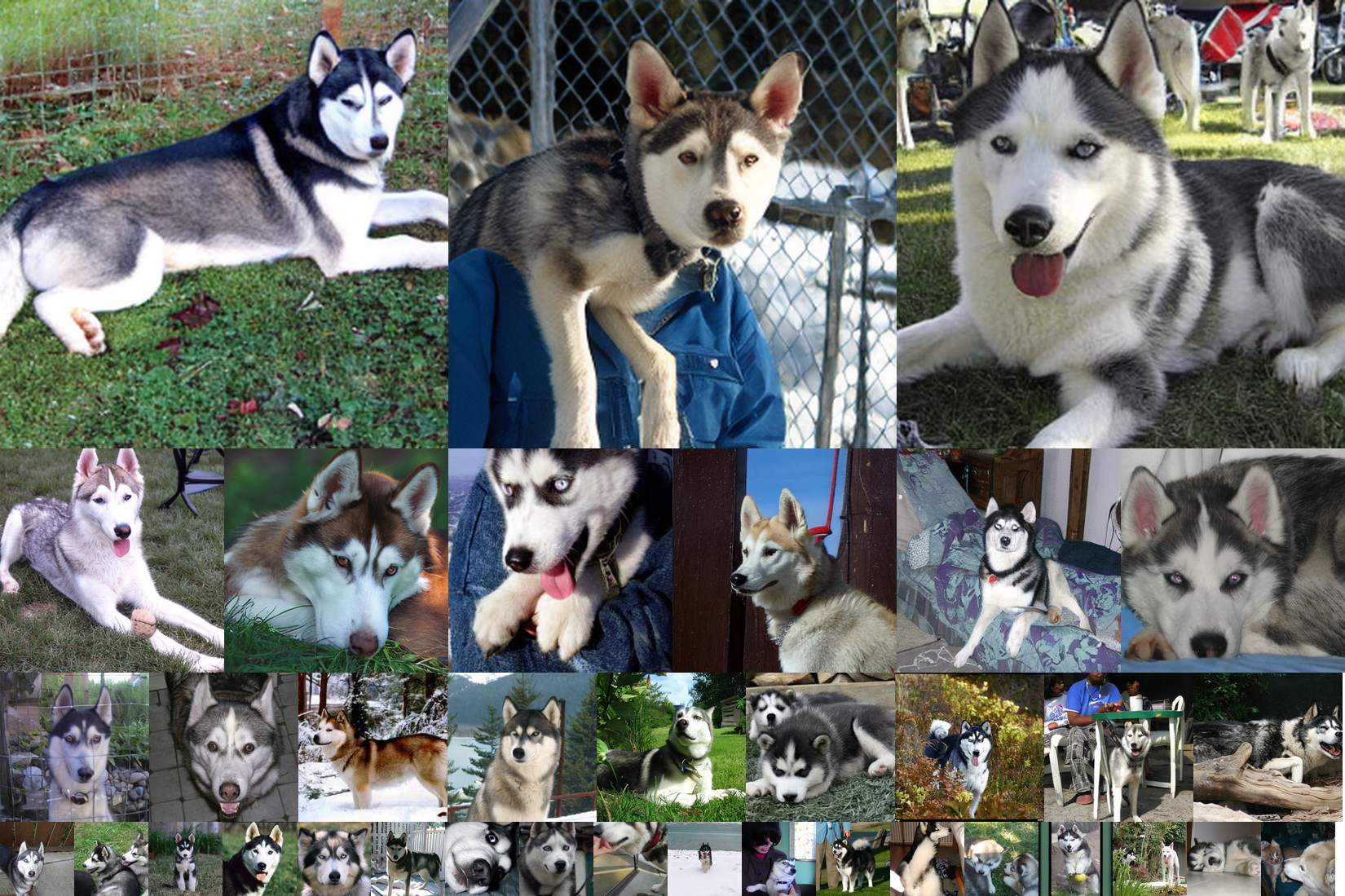}
  \caption{Husky (class~250). Samples from Ours-XL with AutoGuidance, \(w=1.5\).}
  \label{fig:qual-husky}
\end{figure}

\begin{figure}[!htbp]
  \centering
  \includegraphics[width=\linewidth]{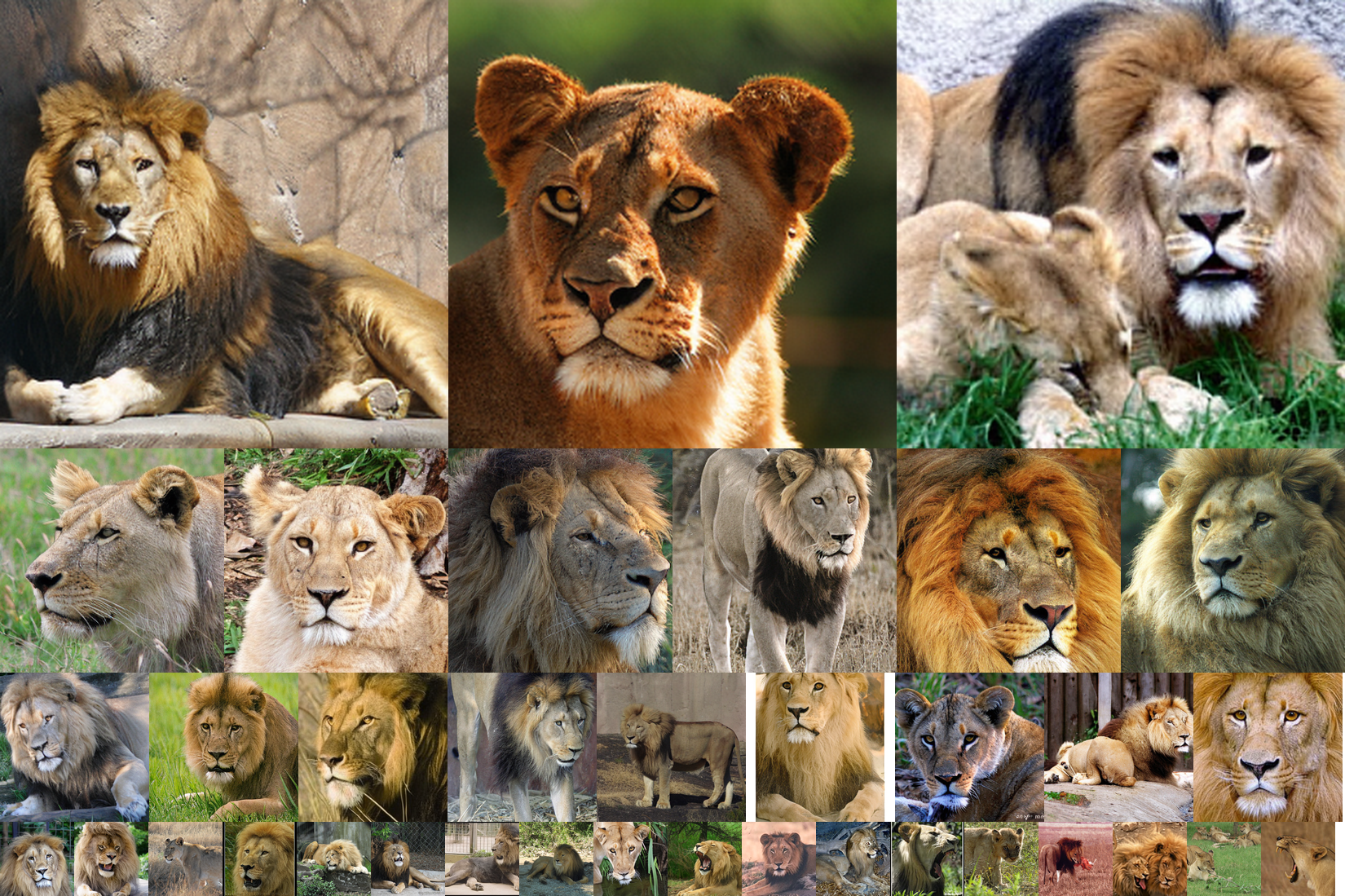}
  \caption{Lion (class~291). Samples from Ours-XL with AutoGuidance, \(w=1.5\).}
  \label{fig:qual-lion}
\end{figure}

\begin{figure}[!htbp]
  \centering
  \includegraphics[width=\linewidth]{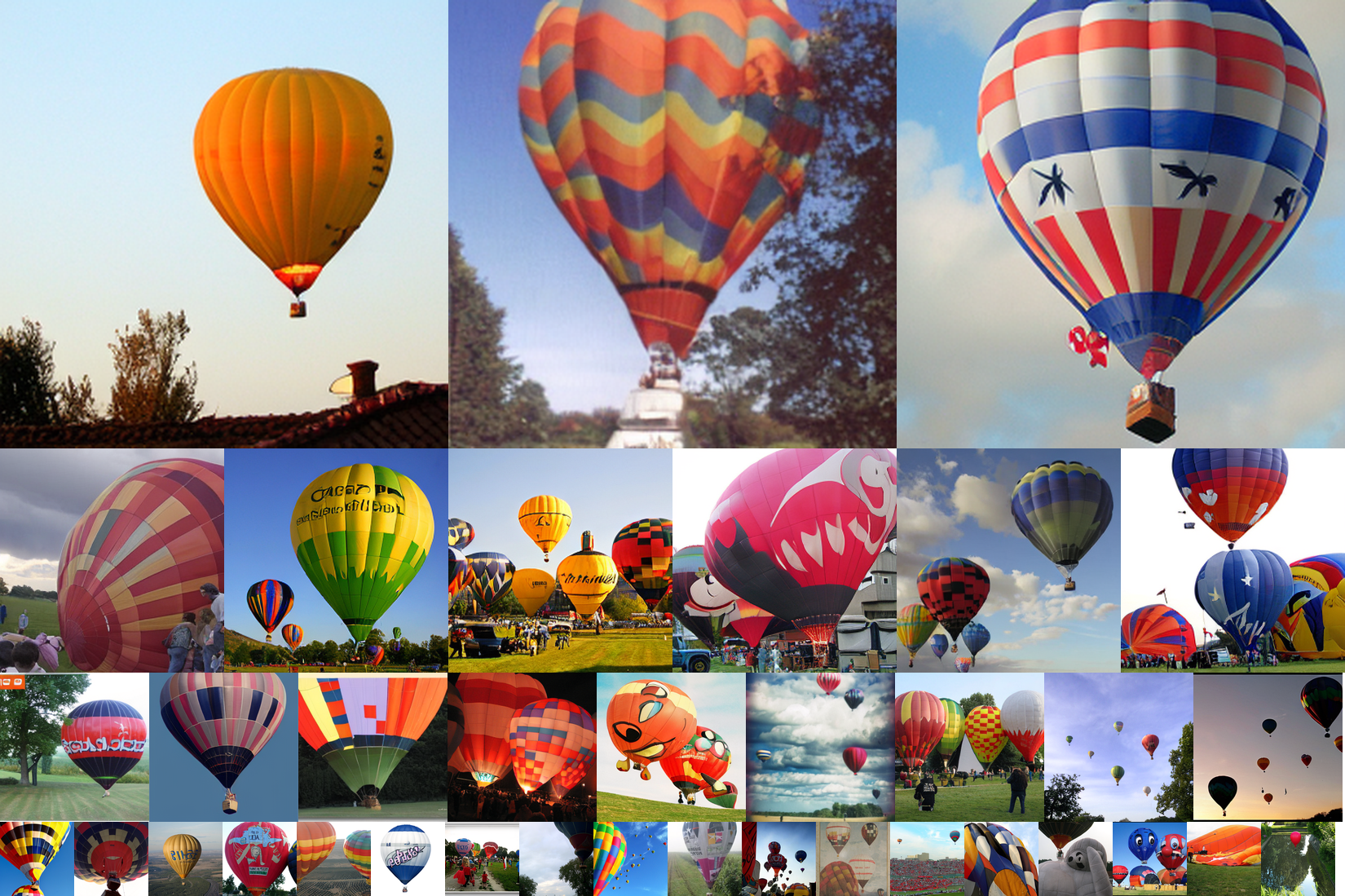}
  \caption{Balloon (class~417). Samples from Ours-XL with AutoGuidance, \(w=1.5\).}
  \label{fig:qual-balloon}
\end{figure}

\begin{figure}[!htbp]
  \centering
  \includegraphics[width=\linewidth]{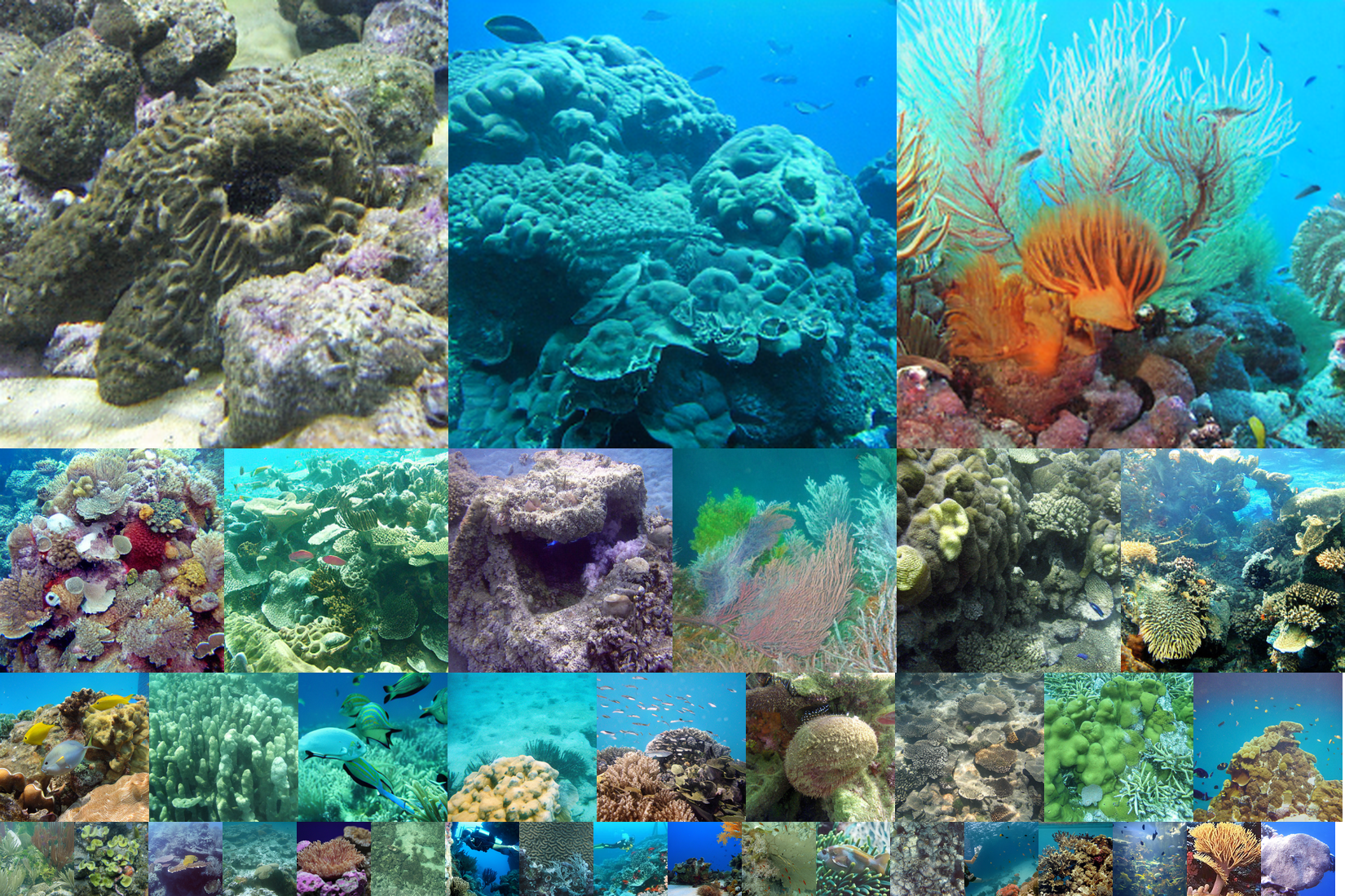}
  \caption{Coral reef (class~973). Samples from Ours-XL with AutoGuidance, \(w=1.5\).}
  \label{fig:qual-coral}
\end{figure}

\begin{figure}[!htbp]
  \centering
  \includegraphics[width=\linewidth]{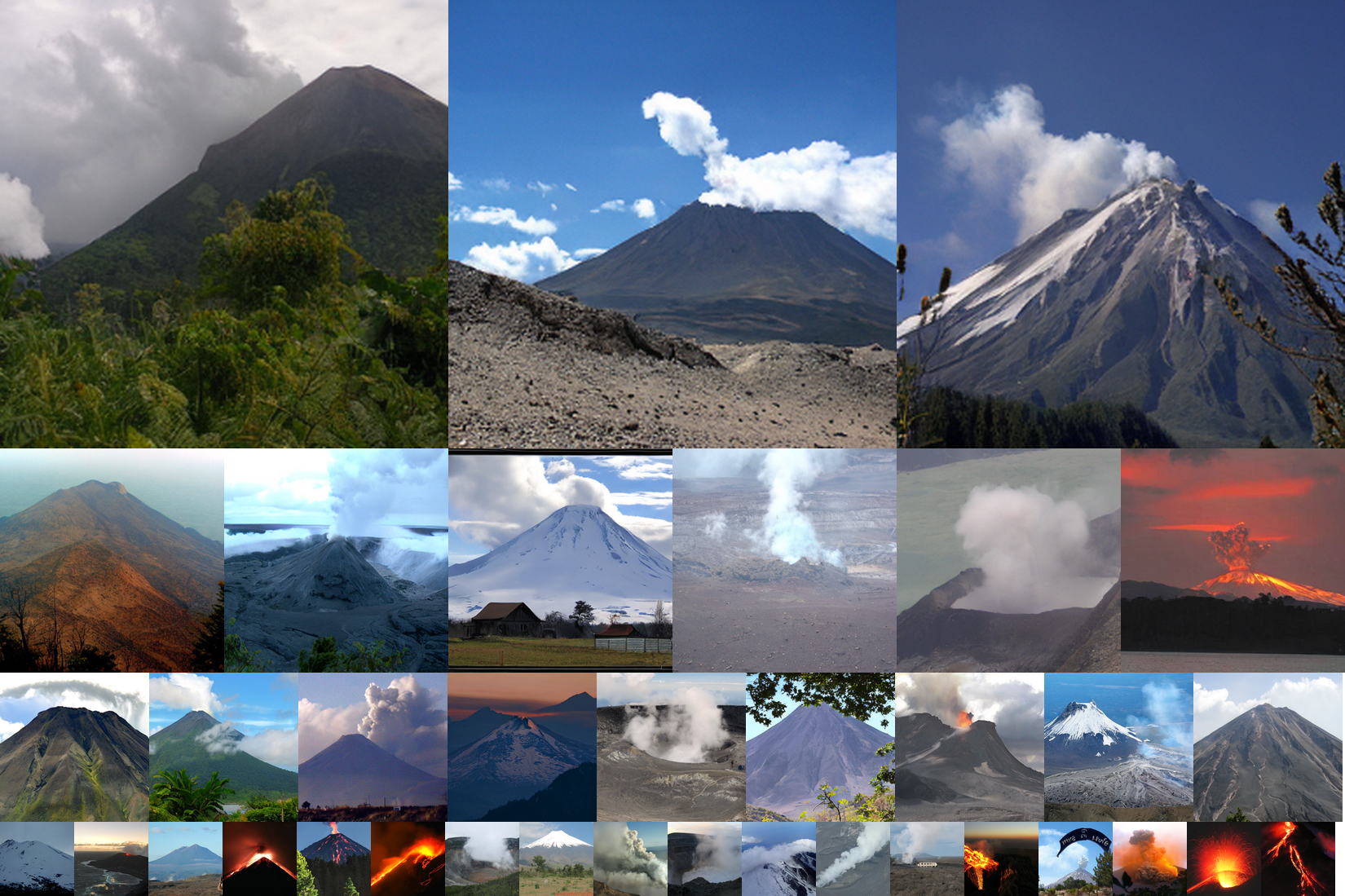}
  \caption{Volcano (class~980). Samples from Ours-XL with AutoGuidance, \(w=1.5\).}
  \label{fig:qual-volcano}
\end{figure}

\end{document}